\documentclass[a4paper, 11pt]{article}
\usepackage{graphicx} 

\usepackage{url}
\usepackage{hyperref}
\usepackage{caption}

\date{} 					

\usepackage{hyphenat} 
\usepackage[textwidth=151mm, hcentering, vmargin={24mm,32.5mm}, foot=20mm]{geometry}

\usepackage{authblk}

\author[1]{Reijo~Jaakkola}
\author[1]{Tomi~Janhunen}
\author[1]{Antti~Kuusisto}
\author[2]{Magdalena~Ortiz}
\author[1]{Matias~Selin}
\author[2]{Mantas~Šimkus}

\affil[1]{Tampere University, Finland}
\affil[2]{TU Wien, Austria}

\usepackage{tikz}
\usepackage{multicol}
\usepackage{array}

\usepackage{arydshln}
\usepackage{makecell}

\usepackage{colonequals}
\usepackage{amsmath, amssymb, amsthm}
\usepackage{mathtools}
\usepackage{thm-restate}

\usepackage{graphicx} 

\usepackage{verbatim}

\usepackage[utf8]{inputenc}
\usepackage[T1]{fontenc}

\usepackage{enumitem}
\usepackage{tikz}
\usepackage{lipsum}
\usepackage{pdfpages}
\usepackage{multicol} 
\usepackage{xspace} 
\usepackage{stmaryrd}

\usepackage{booktabs} 

\usepackage{mathtools}
\usepackage{colonequals}

\usepackage{marginnote}

\usepackage{bm}

\usepackage{arydshln}

\theoremstyle{plain}
\newtheorem{theorem}{Theorem}
\newtheorem{lemma}[theorem]{Lemma}
\newtheorem{proposition}[theorem]{Proposition}

\theoremstyle{definition}

\newtheorem{example}[theorem]{Example}

\title{Graph Learning via Logic-Based Weisfeiler-Leman Variants and Tabularization}

\usepackage{enumitem}
\setitemize{noitemsep,topsep=0pt,parsep=3pt,partopsep=0pt}
\setenumerate{noitemsep,topsep=0pt,parsep=3pt,partopsep=0pt}

\usepackage{amsmath}
\usepackage{graphicx} 
\usepackage{graphics}
\usepackage{url}
\usepackage{xcolor}
\usepackage{amsmath, amssymb, amsthm}
\usepackage{tikz}
\usepackage{pgfplots}
\usepackage{pgfplotstable}
\usepackage{comment}
\usepackage{makecell}
\usepackage{adjustbox}

\usetikzlibrary{positioning, arrows.meta, shapes.geometric, fit, backgrounds, calc}

\theoremstyle{plain}

\newcommand{\GQ}{\langle Q\rangle}               

\pgfplotsset{compat=1.18} 

\begin{document}

\maketitle

\begin{abstract}
We present a novel approach for graph classification based on tabularizing graph data via new variants of the Weisfeiler–Leman algorithm and then applying methods for tabular data.
The variants are obtained by modifying the underlying logical framework, and we establish a precise theoretical characterization of their expressive power using a novel generalization of the bisimulation game for generalized quantifiers. We then test our method on 14 datasets that span a range of application domains. The experiments demonstrate that on datasets with up to 40 000 samples, our approach generally matches the predictive performance of graph neural networks and graph transformers, without requiring a GPU or extensive hyperparameter tuning. Even when our method's tuning time is included and the baselines' is not, our method is 5--20 times faster. When tuning time is included for all methods, the gap is significantly greater in favour of our method.
\end{abstract}

\section{Introduction}\label{sec:intro}

Graph Neural Networks (GNNs) and Graph Transformers (GTs) achieve strong results on graph-structured data by learning from both node features and graph topology \cite{ScarselliTNN2009,xu2018powerful,rampavsek2022recipe}. However, this comes at a cost: training is typically much slower than for state-of-the-art methods on tabular data \cite{kriege2020survey,dwivedi2023benchmarking}, and performance is highly sensitive to architecture choice and hyperparameter tuning \cite{Errica2020A, tonshoff2023did, luocan}.

In this paper, we show that for graph classification on small and medium-sized datasets, much of this complexity can be avoided. Our approach is based on a simple observation: the Weisfeiler--Leman algorithm (WL) \cite{Weisfeiler:1968ve}, originally designed for graph isomorphism testing, computes exactly the node-level information that standard GNNs can learn from \cite{xu2018powerful,GroheLICS2021LogicGNNs}. By running WL and counting the frequency of each resulting node label, we obtain a fixed-length feature vector for each graph, converting the graph dataset into a tabular dataset. We can then apply off-the-shelf tabular classifiers---in our case, random forests \cite{DBLP:journals/ml/Breiman01} with default hyperparameters---directly to this representation.

This naturally raises the question: how much neighbourhood information does the
algorithm actually need?
Standard WL counts neighbours exactly---but does a classifier really benefit from
knowing that a node has 234 versus 235 neighbours of a given type?
We address this by defining a family of WL variants parametrized by a set $\mathcal{Q}$ of
\emph{generalized quantifiers}, which control how coarsely the
algorithm aggregates neighbourhood information.
These range from a very crude variant that tracks only which types are
present in the neighbourhood (ignoring counts entirely) to the standard WL that
counts precisely. Our final pipeline (which we call $\mathcal{Q}$\textbf{-WL-RF}) involves selecting a set of $\mathcal{Q}$-WL variants and running them for a number of iterations, thus producing a \textit{set} of colourings of the input graph. We then include all of these colourings in the produced tabular dataset.

On the theoretical side, we establish an exact characterization of the expressive
power of $\mathcal{Q}$-WL for any set of generalized quantifiers~$\mathcal{Q}$,
via a novel generalization of the bisimulation game.
This provides a precise guarantee: each variant distinguishes exactly those graphs
that are separable in the corresponding modal logic.

For empirical verification of these ideas, we select three variants of $\mathcal{Q}$-WL and test our method on 14 benchmark graph classification datasets. We compare our approach to the other graph classification paradigms: graph kernels, GNNs, and GTs. The empirical results confirm the validity of our approach: we can very efficiently perform tabularization and learn accurate classifiers.

\textbf{The main contributions of the paper can thus be summarized as follows:}
\begin{enumerate}
    \item A general framework ($\mathcal{Q}$-WL) for defining variants of the
    WL algorithm via generalized quantifiers, together with a novel
    generalization of the bisimulation game that exactly characterizes the
    expressive power of any such variant (Theorem~\ref{thm:generalized-quantifiers-characterization}).
    \item Experiments showing that graph classification based on this paradigm ($\mathcal{Q}$-WL-RF) achieves \begin{enumerate}
        \item predictive performance generally matching or exceeding that of graph kernels, GNNs and GTs;
        \item runtimes 5--20 times faster than those of GNNs and GTs, even without taking hyperparameter optimization into account;
        \item memory scaling on dataset size vastly better than that of graph kernels.
    \end{enumerate}
\end{enumerate}


\section{Preliminaries}
\label{sec:preliminaries}

\subsection{Logic}
 
A \textbf{graph} is a pair \(G = (V,E)\), where \(V\) is a \emph{finite} non-empty set and \(E \subseteq V \times V\). We allow directed edges and reflexive loops. Given \(w \in V\), we write \(N(w):=\{w' \in V \mid (w,w') \in E\}\) for the set of (out-)neighbours of \(w\).

Let $\Phi$ be a finite set of proposition symbols.
A \textbf{Kripke model} over~$\Phi$ is a triple $\mathfrak{M} = (V, E, \ell)$, where $(V, E)$ is a graph and $\ell \colon V \to \mathcal{P}(\Phi)$ is a labeling function.
We write $v \in \mathfrak{M}$ for $v \in V$ and call $(\mathfrak{M}, v)$ a \textbf{pointed Kripke model}. Thus a Kripke model over \(\Phi\) is just a labeled graph, labels being subsets of \(\Phi\).

The formulas of \textbf{graded modal logic} (GML) are generated by the following grammar:
\[
  \varphi \;::=\; \bot \mid p \mid \lnot\varphi \mid \varphi \land \varphi \mid \varphi \lor \varphi \mid \Diamond\varphi \mid \Diamond^{\geq k}\varphi
  \qquad (p \in \Phi,\; k \in \mathbb{Z}_+).
\]
We write $\Diamond^{=k}\varphi$ for $\Diamond^{\geq k}\varphi \land \lnot\Diamond^{\geq k+1}\varphi$.
\textbf{Modal logic} (ML) is the fragment without $\Diamond^{\geq k}$,
and \textbf{propositional logic} (PL) is the fragment without any modalities.
The \textbf{modal depth} $\mathrm{md}(\varphi)$ is defined as usual, with $\mathrm{md}(\Diamond\varphi) = \mathrm{md}(\Diamond^{\geq k}\varphi) = 1 + \mathrm{md}(\varphi)$.
We write $\mathrm{GML}^d$ for the set of GML-formulas of modal depth at most~$d$.

Given a Kripke model \(\mathfrak{M}\) and \(v \in \mathfrak{M}\), satisfaction \(\mathfrak{M},v \Vdash \varphi\) for \(\varphi \in \mathrm{GML}\) is defined as usual for the Boolean connectives, with the modal cases:
\begin{itemize}[nosep]
    \item \(\mathfrak{M},v \Vdash \lozenge \varphi \Leftrightarrow \mathfrak{M},u \Vdash \varphi\) for some \(u \in N(v)\),
    \item \(\mathfrak{M},v \Vdash \lozenge^{\geq k} \varphi \Leftrightarrow |\{u \in N(v) \mid \mathfrak{M},u \Vdash \varphi\}| \geq k\).
\end{itemize}

The expressive power of GML can be studied via \textbf{graded bisimulation}, which is a similarity measure between labeled graphs \cite{graded_modal_logic}. Following \cite{otto2023gradedmodallogiccounting}, we formulate it in terms of a game. The \textbf{graded bisimulation game} on the pointed Kripke models $(\mathfrak{M},w)$ and $(\mathfrak{N},w')$ is a two-player pebble game beginning in the position $(w,w')$. Each round proceeds from a position $(v,v')$ as follows:
\begin{enumerate}
    \item Player 1 wins immediately if $v$ and $v'$ do not satisfy the same proposition symbols.
    \item Player 1 chooses either \(X \subseteq N(v)\) or \(X \subseteq N(v')\). Player 2 responds with \(Y \subseteq N(v')\) or \(Y \subseteq N(v)\) respectively such that \(X\) and \(Y\) have the same size.
    \item Player 1 chooses \(u \in Y\) and Player 2 responds with \(u' \in X\). If \(u \in \mathfrak{M}\), the game continues from the position \((u,u')\); otherwise it continues from the position \((u',u)\).
\end{enumerate}
If the game lasts for infinitely many rounds, Player 2 wins. It is well-known that two pointed Kripke models $(\mathfrak{M},w)$ and $(\mathfrak{N},w')$ can be separated by GML if and only if Player \(1\) has a winning strategy in the graded-bisimulation game played on  $(\mathfrak{M},w)$ and $(\mathfrak{N},w')$. See, e.g., \cite{otto2023gradedmodallogiccounting} for the proof.

\subsection{Weisfeiler--Leman and Types}

We start by recalling the well-known (\(1\)-dimensional) Weisfeiler–Leman algorithm (WL) \cite{Weisfeiler:1968ve,DBLP:journals/jmlr/ShervashidzeSLMB11}. Given a labeled graph~$\mathfrak{M}$ and a depth $d \in \mathbb{N}$, WL iteratively refines the labeling of~$\mathfrak{M}$.
First, set $\mathfrak{M}_0 := \mathfrak{M}$.
Then, given~$\mathfrak{M}_i$ with labeling~$\ell_i$, two nodes $u, v$ receive the same label in~$\mathfrak{M}_{i+1}$ if and only if $\ell_i(u) = \ell_i(v)$ and
\[
  \{\!\{ \ell_i(w) \mid w \in N(u) \}\!\} \;=\; \{\!\{ \ell_i(w) \mid w \in N(v) \}\!\},
\]
where $\{\!\{\cdot\}\!\}$ denotes a multiset.
The process stops at~$\mathfrak{M}_d$.

The labels produced by WL correspond to \textbf{graded modal $d$-types}: $\mathrm{GML}^d$-formulas belonging to the set $\mathcal{T}^d_g$, which defined as follows. First, \(\mathcal{T}_g^0\) is the set of all formulas of the form
\[\bigwedge_{p \in \Psi} p \ \land \ \bigwedge_{p \not\in \Psi} \neg p,\]
where \(\Psi \subseteq \Phi\). We then define \(\mathcal{T}_g^{d + 1}\) as the set of all formulas of the form
\[
    \sigma \land \bigwedge_{\tau \in I} \lozenge^{=\mathrm{count}(\tau)} \tau \ \land \ \square \bigvee_{\tau \in I} \tau,
\]
where \(\sigma \in \mathcal{T}_g^0\), \(I \subseteq \mathcal{T}_g^d\) is a finite set and \(\mathrm{count}:I \to \mathbb{Z}_+\) is a function.

\begin{example}
Consider the following Kripke model (i.e. graph):
\begin{center}
\begin{tikzpicture}[node distance=1.1cm, world/.style={circle,draw,inner sep=2pt,minimum size=8mm}]
  \node[world, label=left:{\small 4}] (4) at (1.33,1.33) {$q$};
  \node[world, label=left:{\small 1}] (1) {$p$};
  \node[world, label=below:{\small 2}] (2) at (1.33,0) {$p$};
  \node[world, label=right:{\small 3}] (3) at (2.66,0) {$q$};
  \draw (1)--(2)--(3);
  \draw (4)--(2);
\end{tikzpicture}
\end{center}
\noindent The graded modal \(0\)-types partition the nodes by their labels: nodes \(1, 2\) have $0$-type \(p\) and nodes \(3, 4\) have $0$-type \(q\). After one round of WL, three distinct \(1\)-types emerge: Node \(1\) has $1$-type \( p \land \lozenge^{=1}p \land \square p\), node \(2\) has $1$-type \(p \land \lozenge^{=1}p\land\lozenge^{=2}q\land\square (p\lor q)\), and nodes $3$ and $4$ have $1$-type \(q \land \lozenge^{=1}p\land\square p\).
\end{example}

Every pointed Kripke model satisfies a unique graded modal $d$-type, and two pointed models satisfy the same $d$-type if and only if they agree on all of~$\mathrm{GML}^d$~\cite{otto2023gradedmodallogiccounting}.
Since at round~$d$ the WL algorithm computes the graded $d$-type of each node, WL-equivalence after $d$~rounds coincides with $\mathrm{GML}^d$-equivalence. In particular, since separation in WL and GML is always realized at some finite depth $d$, we can abstract away the specific $d$ and obtain the following.

\begin{proposition}\label{prop:theoretical-guarantee}
Let $(\mathfrak{M},w)$ and $(\mathfrak{N},w')$ be pointed Kripke models. The graded-bisimulation game is sound and complete in the sense that the following are equivalent:
\begin{enumerate}
    \item Player 1 has a winning strategy in the graded-bisimulation game played on  $(\mathfrak{M},w)$ and $(\mathfrak{N},w')$.
    \item $(\mathfrak{M},w)$ and $(\mathfrak{N},w')$ are separated by \(\mathrm{GML}\).
    \item $(\mathfrak{M},w)$ and $(\mathfrak{N},w')$ are separated by \(\mathrm{WL}\).
\end{enumerate}
\end{proposition}


\section{Method}
\label{sec:method-description}

\begin{figure*}[t]
\centering
\begin{tikzpicture}[scale=0.65, transform shape,
    node distance=0.8cm and 0.6cm,
    >={Stealth[length=2mm]},
    box/.style={draw, rounded corners, minimum height=0.9cm, minimum width=1.8cm, 
            align=center, font=\small},
    methodbox/.style={box, fill=blue!10, minimum width=1.8cm},
    databox/.style={box, fill=orange!12, minimum width=1.6cm},
    resultbox/.style={box, fill=green!12, minimum width=1.4cm},
    graphnode/.style={circle, draw, fill=gray!30, inner sep=1.5pt, minimum size=4pt},
    arrow/.style={->, thick, gray!70!black},
    sectionlabel/.style={font=\large\bfseries, text=black},
    dataset/.style={draw, dashed, rounded corners=3pt, inner sep=4pt, gray!70}
]

\node[sectionlabel] (gk-label) at (1, 0) {Weisfeiler--Leman Graph Kernel (WL-GK)};

\node[methodbox] (gk-wl) at (-4.5, -1.5) {\makecell{$d$ rounds\\of WL}};

\begin{scope}[shift={(-2.3, -1.5)}]
    \node[graphnode, fill=violet!70] (t1a) at (0,0.2) {};
    \node[graphnode, fill=orange!70] (t1b) at (0.35,0.45) {};
    \node[graphnode, fill=cyan!70] (t1c) at (0.4,-0.05) {};
    \node[graphnode, fill=pink!70] (t1d) at (0.75,0.25) {};
    \draw[thick] (t1a)--(t1b)--(t1d)--(t1c)--(t1a) (t1b)--(t1c);
    \node[graphnode, fill=yellow!70] (t2a) at (1.1,0.15) {};
    \node[graphnode, fill=lime!70] (t2b) at (1.4,0.45) {};
    \node[graphnode, fill=teal!70] (t2c) at (1.7,0.15) {};
    \draw[thick] (t2a)--(t2b)--(t2c)--(t2a);
    \node[graphnode, fill=brown!70] (t3a) at (0.4,-0.55) {};
    \node[graphnode, fill=magenta!70] (t3b) at (0.7,-0.35) {};
    \node[graphnode, fill=olive!70] (t3c) at (1.0,-0.55) {};
    \node[graphnode, fill=purple!70] (t3d) at (1.3,-0.35) {};
    \draw[thick] (t3a)--(t3b)--(t3c)--(t3d) (t3b)--(t3d);
    \node[dataset, fit=(t1a)(t1b)(t2b)(t2c)(t3a)(t3d), inner sep=6pt] {};
\end{scope}
\node[font=\scriptsize] at (-1.5, -2.6) {Refined node labels};

\node[methodbox] (kernel) at (1.9, -1.5) {Compute similarities\\for all graph pairs};

\node[databox, minimum width=2.4cm, minimum height=1.4cm] (kmatrix) at (5.2, -1.5) {};
\node[font=\tiny] at (5.2, -1.5) {$\begin{bmatrix} 0.67 & \cdots & 0.20 \\ \vdots & \ddots & \vdots \\ 0.96 & \dots & 0.11 \end{bmatrix}$};
\node[font=\scriptsize] at (5.2, -2.4) {Similarity matrix};

\node[methodbox, fill=purple!10] (svm) at (8.0, -1.5) {SVM};

\node[resultbox] (gk-class) at (10.5, -1.5) {Class\\labels};

\draw[arrow] (-6.65, -4.5) -- (gk-wl);
\draw[arrow] (gk-wl) -- (-2.6, -1.5);
\draw[arrow] (-0.25, -1.5) -- (kernel);
\draw[arrow] (kernel) -- (kmatrix);
\draw[arrow] (kmatrix) -- (svm);
\draw[arrow] (svm) -- (gk-class);

\node[sectionlabel] (gnn-label) at (1, -3.5) {Graph Neural Networks (GNNs) / Graph Transformers (GTs)};

\begin{scope}[shift={(-8.7, -5)}]
    \node[graphnode, fill=red!40] (gnn-1a) at (0,0.2) {};
    \node[graphnode, fill=blue!40] (gnn-1b) at (0.35,0.45) {};
    \node[graphnode, fill=red!40] (gnn-1c) at (0.4,-0.05) {};
    \node[graphnode, fill=green!40] (gnn-1d) at (0.75,0.25) {};
    \draw[thick] (gnn-1a)--(gnn-1b)--(gnn-1d)--(gnn-1c)--(gnn-1a) (gnn-1b)--(gnn-1c);
    \node[graphnode, fill=blue!40] (gnn-2a) at (1.1,0.15) {};
    \node[graphnode, fill=green!40] (gnn-2b) at (1.4,0.45) {};
    \node[graphnode, fill=red!40] (gnn-2c) at (1.7,0.15) {};
    \draw[thick] (gnn-2a)--(gnn-2b)--(gnn-2c)--(gnn-2a);
    \node[graphnode, fill=green!40] (gnn-3a) at (0.4,-0.55) {};
    \node[graphnode, fill=blue!40] (gnn-3b) at (0.7,-0.35) {};
    \node[graphnode, fill=red!40] (gnn-3c) at (1.0,-0.55) {};
    \node[graphnode, fill=blue!40] (gnn-3d) at (1.3,-0.35) {};
    \draw[thick] (gnn-3a)--(gnn-3b)--(gnn-3c)--(gnn-3d) (gnn-3b)--(gnn-3d);
    \node[dataset, fit=(gnn-1a)(gnn-1b)(gnn-2b)(gnn-2c)(gnn-3a)(gnn-3d), inner sep=6pt] {};
\end{scope}
\node[font=\small] at (-7.8, -6.2) {Set of $n$ graphs};

\node[methodbox, minimum width=2.2cm] (mp) at (-4.5, -5) {$L$ rounds of message\\passing / attention};

\begin{scope}[shift={(-1.8, -5)}]
    \node[graphnode, fill=violet!70] (mp1a) at (0,0.2) {};
    \node[graphnode, fill=orange!70] (mp1b) at (0.35,0.45) {};
    \node[graphnode, fill=cyan!70] (mp1c) at (0.4,-0.05) {};
    \node[graphnode, fill=pink!70] (mp1d) at (0.75,0.25) {};
    \draw[thick] (mp1a)--(mp1b)--(mp1d)--(mp1c)--(mp1a) (mp1b)--(mp1c);
    \node[graphnode, fill=yellow!70] (mp2a) at (1.1,0.15) {};
    \node[graphnode, fill=lime!70] (mp2b) at (1.4,0.45) {};
    \node[graphnode, fill=teal!70] (mp2c) at (1.7,0.15) {};
    \draw[thick] (mp2a)--(mp2b)--(mp2c)--(mp2a);
    \node[graphnode, fill=brown!70] (mp3a) at (0.4,-0.55) {};
    \node[graphnode, fill=magenta!70] (mp3b) at (0.7,-0.35) {};
    \node[graphnode, fill=olive!70] (mp3c) at (1.0,-0.55) {};
    \node[graphnode, fill=purple!70] (mp3d) at (1.3,-0.35) {};
    \draw[thick] (mp3a)--(mp3b)--(mp3c)--(mp3d) (mp3b)--(mp3d);
    \node[dataset, fit=(mp1a)(mp1b)(mp2b)(mp2c)(mp3a)(mp3d), inner sep=6pt] {};
\end{scope}
\node[font=\scriptsize] at (-0.94, -6.15) {Updated node embeddings};

\node[methodbox] (pool) at (2.2, -5) {Pooling};

\node[databox, minimum width=1.8cm, minimum height=1.0cm] (gembed) at (5.2, -5) {$\mathbf{h}_{G_1}, \ldots, \mathbf{h}_{G_n}$};
\node[font=\scriptsize] at (5.2, -5.9) {\makecell{High-dimensional\\graph embeddings}};

\node[methodbox, fill=purple!10, minimum width=1.2cm] (mlp) at (8.0, -5) {MLP};

\node[resultbox] (gnn-class) at (10.5, -5) {Class\\labels};

\draw[arrow] (-6.65, -5) -- (mp);
\draw[arrow] (mp) -- (-2.1, -5);
\draw[arrow] (0.25, -5) -- (pool);
\draw[arrow] (pool) -- (gembed);
\draw[arrow] (gembed) -- (mlp);
\draw[arrow] (mlp) -- (gnn-class);

  \node[sectionlabel] (wl-label) at (1, -7) {$\mathcal{Q}$-WL-RF (Ours)};

  \node[methodbox]
      (wl) at (-4.5, -8.5) {\makecell{$d_v$ rounds of $\mathcal{Q}$-WL\\for each variant $v$}};

  \begin{scope}[shift={(-1.8, -8.5)}]
      \begin{scope}[shift={(0.18, -0.18)}, opacity=0.35]
      \node[graphnode, fill=violet!70] (g1a) at (0,0.2) {};
      \node[graphnode, fill=orange!70] (g1b) at (0.35,0.45) {};
      \node[graphnode, fill=cyan!70] (g1c) at (0.4,-0.05) {};
      \node[graphnode, fill=pink!70] (g1d) at (0.75,0.25) {};
      \draw[thick] (g1a)--(g1b)--(g1d)--(g1c)--(g1a) (g1b)--(g1c);
      \node[graphnode, fill=yellow!70] (g2a) at (1.1,0.15) {};
      \node[graphnode, fill=lime!70] (g2b) at (1.4,0.45) {};
      \node[graphnode, fill=teal!70] (g2c) at (1.7,0.15) {};
      \draw[thick] (g2a)--(g2b)--(g2c)--(g2a);
      \node[graphnode, fill=brown!70] (g3a) at (0.4,-0.55) {};
      \node[graphnode, fill=magenta!70] (g3b) at (0.7,-0.35) {};
      \node[graphnode, fill=olive!70] (g3c) at (1.0,-0.55) {};
      \node[graphnode, fill=purple!70] (g3d) at (1.3,-0.35) {};
      \draw[thick] (g3a)--(g3b)--(g3c)--(g3d) (g3b)--(g3d);
      \node[dataset, fit=(g1a)(g1b)(g2b)(g2c)(g3a)(g3d), inner sep=6pt] {};
      \end{scope}
      \begin{scope}[shift={(0.09, -0.09)}, opacity=0.6]
      \node[graphnode, fill=violet!70] (h1a) at (0,0.2) {};
      \node[graphnode, fill=orange!70] (h1b) at (0.35,0.45) {};
      \node[graphnode, fill=cyan!70] (h1c) at (0.4,-0.05) {};
      \node[graphnode, fill=pink!70] (h1d) at (0.75,0.25) {};
      \draw[thick] (h1a)--(h1b)--(h1d)--(h1c)--(h1a) (h1b)--(h1c);
      \node[graphnode, fill=yellow!70] (h2a) at (1.1,0.15) {};
      \node[graphnode, fill=lime!70] (h2b) at (1.4,0.45) {};
      \node[graphnode, fill=teal!70] (h2c) at (1.7,0.15) {};
      \draw[thick] (h2a)--(h2b)--(h2c)--(h2a);
      \node[graphnode, fill=brown!70] (h3a) at (0.4,-0.55) {};
      \node[graphnode, fill=magenta!70] (h3b) at (0.7,-0.35) {};
      \node[graphnode, fill=olive!70] (h3c) at (1.0,-0.55) {};
      \node[graphnode, fill=purple!70] (h3d) at (1.3,-0.35) {};
      \draw[thick] (h3a)--(h3b)--(h3c)--(h3d) (h3b)--(h3d);
      \node[dataset, fit=(h1a)(h1b)(h2b)(h2c)(h3a)(h3d), inner sep=6pt] {};
      \end{scope}
      \node[graphnode, fill=violet!70] (cwl1a) at (0,0.2) {};
      \node[graphnode, fill=orange!70] (cwl1b) at (0.35,0.45) {};
      \node[graphnode, fill=cyan!70] (cwl1c) at (0.4,-0.05) {};
      \node[graphnode, fill=pink!70] (cwl1d) at (0.75,0.25) {};
      \draw[thick] (cwl1a)--(cwl1b)--(cwl1d)--(cwl1c)--(cwl1a) (cwl1b)--(cwl1c);
      \node[graphnode, fill=yellow!70] (cwl2a) at (1.1,0.15) {};
      \node[graphnode, fill=lime!70] (cwl2b) at (1.4,0.45) {};
      \node[graphnode, fill=teal!70] (cwl2c) at (1.7,0.15) {};
      \draw[thick] (cwl2a)--(cwl2b)--(cwl2c)--(cwl2a);
      \node[graphnode, fill=brown!70] (cwl3a) at (0.4,-0.55) {};
      \node[graphnode, fill=magenta!70] (cwl3b) at (0.7,-0.35) {};
      \node[graphnode, fill=olive!70] (cwl3c) at (1.0,-0.55) {};
      \node[graphnode, fill=purple!70] (cwl3d) at (1.3,-0.35) {};
      \draw[thick] (cwl3a)--(cwl3b)--(cwl3c)--(cwl3d) (cwl3b)--(cwl3d);
      \node[dataset, fit=(cwl1a)(cwl1b)(cwl2b)(cwl2c)(cwl3a)(cwl3d), inner sep=6pt] {};
  \end{scope}
  \node[font=\scriptsize] at (-1.0, -9.9) {$\mathcal{Q}$-refined labels (per variant)};

  \node[methodbox] (count) at (2.2, -8.5) {Count types\\\& concatenate};

\node[databox, minimum width=2.6cm, minimum height=1.6cm] (table) at (5.2, -8.5) {};
\node[font=\tiny] at (5.2, -8.5) {%
$\begin{array}{c|ccc}
 & \tau_1 & \cdots & \tau_k \\
\hline
G_1 & 3 & \cdots & 1\\
\vdots & \vdots & \ddots & \vdots \\
G_n & 0 & \cdots & 2
\end{array}$};
\node[font=\scriptsize] at (5.2, -9.5) {Tabular dataset};

  \node[methodbox, fill=purple!10, minimum width=1.6cm] (rf) at (8.0, -8.5) {Random\\Forest};

  \node[resultbox] (wl-class) at (10.5, -8.5) {Class\\labels};

  \draw[arrow] (-6.65, -5.5) -- (wl);
  \draw[arrow] (wl) -- (-2.1, -8.5);
  \draw[arrow] (0.25, -8.5) -- (count);
  \draw[arrow] (count) -- (table);
  \draw[arrow] (table) -- (rf);
  \draw[arrow] (rf) -- (wl-class);

\end{tikzpicture}
\caption{Comparison of three graph classification paradigms.
}
\label{fig:paradigms}
\end{figure*}
 
Our method, $\mathcal{Q}$\textbf{-WL-RF}, is compared to the other graph classification paradigms in Figure \ref{fig:paradigms}. Below, we give a step-by-step description of how it works. 


\subsection{Cruder variants of WL}\label{subsec:crude-wl}

As described in Section \ref{sec:intro}, a potential drawback of standard WL (which we call \textbf{Full WL}) is that the labels it produces can be too fine-grained.
Since Full WL counts neighbours exactly, it distinguishes two nodes that differ only in having, say, 234 versus 235 neighbours of a given type---a distinction unlikely to be useful for classification.
On high-degree nodes, this can lead to an explosion of rare types that hurts generalization.

To mitigate this, we propose two variants of the WL-algorithm which use cruder forms of counting. Both variants follow the same iterative refinement structure: at each round, two nodes receive the same label if and only if they had the same label in the previous round and the \emph{modalities in use} agree about their neighbourhoods.
\begin{enumerate}
    \item \textbf{Plain WL}. Corresponds to plain modal logic ML and captures only \textit{which} types are realized in the neighbourhood, without any counting information.
    \item \textbf{Majority WL}. Extends Plain WL by also separately capturing those types that are realized in a strict majority of the neighbourhood.
\end{enumerate}

The general framework unifying these variants is formalized in Section~\ref{sec:arbitrary-quantifiers}, in which we give an analogue of Proposition \ref{prop:theoretical-guarantee}, providing a formal guarantee that the WL-variants distinguish exactly those nodes separable in their corresponding modal logics.

\subsection{Tabularization via type counting}

Let $\mathcal{D} = \{(G_1, y_1), \ldots, (G_n, y_n)\}$ be a set of labeled training
graphs and $\mathcal{D}_{\mathrm{test}}=\{G'_1,\dots,G'_m\}$ be a set of test graphs.
$\mathcal{Q}$-WL-RF combines features from a fixed set of $\mathcal{Q}$-WL
variants, each corresponding to a different choice of generalized quantifiers
in the underlying logic. Let $\mathcal{V}$ be this set of variants;
in our experiments,
$\mathcal{V} = \{\text{Full}, \text{Plain}, \text{Majority}\}$.
For each variant $v \in \mathcal{V}$, we independently select a refinement
depth $d_v \in \mathbb{N}$ by validation-set performance, and then run the
three steps below.

\paragraph{Step 1: Per-variant $\mathcal{Q}$-WL refinement.}
For each variant $v \in \mathcal{V}$, we run $d_v$~rounds of the corresponding
$\mathcal{Q}$-WL refinement on all training and test graphs.
Each node $u$ thus obtains a label $\ell_v^{(d_v)}(u)$ corresponding to
its graded $\mathcal{Q}$-modal $d_v$-type under variant $v$.

\paragraph{Step 2: Counting and concatenation.}
For each variant $v \in \mathcal{V}$, let
$\mathcal{T}_v = \{\tau^v_1, \ldots, \tau^v_{k_v}\}$
be the set of distinct node labels appearing in the training graphs after
refinement.
For each graph~$G$ (training or test), we define its variant-$v$ feature
vector $\mathbf{x}^v_G \in \mathbb{N}^{k_v}$ by
\[
\mathbf{x}^v_G[\tau] \;:=\; |\{u \in G \mid \ell_v^{(d_v)}(u) = \tau\}|
\qquad \text{for each } \tau \in \mathcal{T}_v.
\]
Types that appear in test graphs but not in any training graph are discarded,
since the classifier has no training signal for them.
The final feature vector of $G$ is the concatenation
\[
\mathbf{x}_G \;:=\; \bigl( \mathbf{x}^{v_1}_G \,\big|\, \mathbf{x}^{v_2}_G
  \,\big|\, \cdots \,\big|\, \mathbf{x}^{v_{|\mathcal{V}|}}_G \bigr)
\;\in\; \mathbb{N}^{k},
\]
where $k = \sum_{v \in \mathcal{V}} k_v$.
The result is a standard tabular dataset with $k$ features, $n$ training rows
and $m$ test rows.

\paragraph{Step 3: Classification.}
We train a tabular classifier on the vectors $\{(\mathbf{x}_{G_i}, y_i)\}_{i=1}^n$
and use it to predict labels for the test graphs.
We use random forests~\cite{DBLP:journals/ml/Breiman01} with default
hyperparameters, as they are robust to high-dimensional features and perform well without tuning.
However, the framework is modular: the random forest can be replaced by any
tabular classifier without modifying the other phases.

Unlike some graph kernel approaches, we do not use these feature vectors to compute a similarity metric between graphs, but rather provide them directly to the classifier. Since we do not compute similarities for all graph pairs, which takes quadratic time with respect to the number of graphs, our preprocessing step is much more efficient on large datasets (see Figure \ref{fig:memory}).


\section{Generalizing Weisfeiler–Leman with Arbitrary Quantifiers}
\label{sec:arbitrary-quantifiers}

We now formalize the variants from Section~\ref{subsec:crude-wl} within a single framework
and establish an exact characterization of the expressive power of the resulting algorithms. This covers a \textit{very comprehensive} collection of ways of generalizing WL.

A \textbf{generalized quantifier} is an isomorphism-closed class \(Q\) of structures \((D,P)\) with domain \(D\) and a unary relation \(P\subseteq D\). Each generalized quantifier \(Q\) gives rise to a \textbf{generalized modality} \(\langle Q \rangle\). A formula of the form \(\langle Q \rangle \varphi\) is interpreted in a pointed Kripke model such that $\mathfrak{M}, v \Vdash \GQ \varphi$ if and only if $(N(v),\{u \mid u\in N(v) \text{ and } \mathfrak{M},u\Vdash \varphi\})\in Q.$
We use PL(\(\mathcal{Q}\)) to denote the extension of propositional logic with the generalized modalities from a collection \(\mathcal{Q}\) of 
generalized quantifiers.

\begin{example}\label{ex:quantifiers}
  The diamond~$\Diamond$ corresponds to $\langle \exists \rangle$ where
  $\exists = \{(D, P) \mid P \neq \varnothing\}$.
  The counting modality $\Diamond^{\geq k}$ corresponds to $\langle \exists^{\geq k} \rangle$ where
  $\exists^{\geq k} = \{(D, P) \mid |P| \geq k\}$.
  Majority WL uses the majority modality $\langle \mathrm{Maj} \rangle$, where
  $\mathrm{Maj} = \{(D, P) \mid |P| > |D|/2\}$.
  Full WL corresponds to $\mathcal{Q}_{\mathrm{graded}} = \{\exists^{\geq k} \mid k \in \mathbb{N}\}$.
\end{example}

Recall that Full WL can separate pointed structures that can be distinguished in GML. We will next generalize this by describing how to modify the WL-algorithm so that it captures the expressive power of PL(\(\mathcal{Q}\)). The resulting WL-algorithm is called \(\mathcal{Q}\)\textbf{-WL}.

Let $(V_1,E_1,\ell_1)$ and $(V_2,E_2,\ell_2)$ be labeled graphs and $u\in V_1$ and $v\in V_2$. We say that $u$ and $v$ are $\mathcal{Q}$\textbf{-equivalent} if $\ell_1(u)=\ell_2(v)$ and for every \(Q \in \mathcal{Q}\) and every \(C \subseteq \mathrm{range}(\ell_1) \cup \mathrm{range}(\ell_2)\), we have 
\[
  (N(v_1),\; \{u \in N(v_1) \mid \ell_1(u) \in C\}) \in Q
  \;\;\Longleftrightarrow\;\;
  (N(v_2),\; \{u \in N(v_2) \mid \ell_2(u) \in C\}) \in Q.
\]

As an input, \(\mathcal{Q}\)-WL receives a labeled graph \(\mathfrak{M}\) and a depth \(d\in \mathbb{N}\). First, set $\mathfrak{M}_0 :=  \mathfrak{M}$. Then, given $\mathfrak{M}_i$ with labeling $\ell_i$, two nodes receive the same label in $\mathfrak{M}_{i+1}$ if and only if the nodes are \(\mathcal{Q}\)-equivalent in $\mathfrak{M}_i$. The process stops at $\mathfrak{M}_d$.

When $\mathcal{Q}$ is infinite and unrestricted, $\mathcal{Q}$-WL may be uncomputable. In Appendix \ref{appendix:effectively-finite}, we define what it means for $\mathcal{Q}$ to be \textbf{effectively finite}: intuitively, noting that for nodes with at most $N$ neighbours only finitely many modalities are non-equivalent, we require a Turing machine that can identify such a finite set for any $N$. As an example, $\mathcal{Q}_\text{graded}:=\{\exists^{\geq k}\mid k\in\mathbb{N}\}$ is effectively finite, which is expected since $\mathcal{Q}_{\text{graded}}$-WL is just Full WL.

We now define how we extend the bisimulation game to characterize $\mathrm{PL}(\mathcal{Q})$. The $\mathcal{Q}$\textbf{-bisimulation game} on the pointed Kripke models $(\mathfrak{M},w)$ and $(\mathfrak{N},w')$ is a two-player pebble game beginning in the position $(w,w')$. 
Player 1 starts as the \textbf{attacker} and Player 2 as the \textbf{defender}. Each round proceeds from a position $(v,v')$ as follows:
\begin{enumerate}
    \item The attacker wins immediately if $v$ and $v'$ do not satisfy the same proposition symbols.
    \item The attacker chooses a quantifier $Q \in \mathcal{Q}$ and one of the two positions; without loss of generality, assume the attacker chooses $v$. The attacker then chooses a \textbf{witness set} $X\subseteq N(v)$ such that $(N(v),X)\in Q$. The attacker also chooses a \textbf{spillover set} $P\subseteq N(v')$. \textit{Intuitively, $X$ consists of nodes that satisfy some formula $\psi$ and $P$ consists of those types that imply $\psi$ but are not realized in $N(v)$.}
    \item The defender chooses a corresponding witness set $X'\subseteq N(v')$ such that $(N(v'),X')\in Q$ and $P\subseteq X'$.
    \item The attacker chooses one of the following:
    \begin{itemize}
        \item Choose $u\in N(v')\setminus X'$ and $u'\in X$. Swap attacker and defender roles between the players, then a new round begins from the position $(u,u')$.
        \item Choose $u\in X'$, after which the defender chooses $u'\in X\cup P$. Then a new round begins from the position $(u,u')$.
    \end{itemize}
\end{enumerate}

At any point after a set is chosen, the opposing player may instead contest it as follows:
\begin{itemize}
    \item After a witness set $Y\subseteq N(z)$ is chosen, contest that it breaks equivalence by choosing $u\in Y$ and $u'\in N(z) \setminus Y$. If the contesting player is the attacker, swap attacker and defender roles between the players. Then a new round begins from the position $(u,u')$.
    \item After a spillover set $P\subseteq N(v')$ is chosen, contest that it contains a type realized in $N(v)$ by choosing $u\in P$ and $u'\in N(v)$. Then a new round begins from the position $(u,u')$.
\end{itemize}

If at any point a player has no legal moves, then that player immediately loses the game.

\begin{theorem}\label{thm:generalized-quantifiers-characterization}
    Let \((\mathfrak{M},w)\) and $(\mathfrak{N},w')$ be pointed Kripke models. Then the $\mathcal{Q}$-bisimulation game is sound and complete in the sense that the following are equivalent:
    \begin{enumerate}
        \item Player 1 has a winning strategy in the $\mathcal{Q}$-bisimulation game played on $(\mathfrak{M},w)$ and $(\mathfrak{N},w')$.
        \item $(\mathfrak{M},w)$ and $(\mathfrak{N},w')$ are separated by \(\mathrm{PL}(\mathcal{Q})\).
        \item $(\mathfrak{M},w)$ and $(\mathfrak{N},w')$ are separated by $\mathcal{Q}$-$\mathrm{WL}$.
    \end{enumerate}
\end{theorem}
\begin{proof}[Proof sketch.]
The proof is an induction on $d$, which represents the number of rounds of the $\mathcal{Q}$-bisimulation game, the modal depth of the $\mathrm{PL}(\mathcal{Q})$-formula and the number of rounds of $\mathcal{Q}$-WL (corresponding to the node satisfying a \textbf{$\mathcal{Q}$-modal $d$-type}, a generalization of the graded modal $d$-type). The base case is clear, since 1., 2. and 3. all essentially state that $w$ and $w'$ satisfy the same proposition symbols. The induction case requires more work and is proved in Appendix \ref{generalized-quantifiers-proof} via three technical Lemmas.
\end{proof}
Our definition of a generalized quantifier can be extended further in two ways. First, we may allow them to have an arbitrary (but finite) \textbf{quantifier width}, covering, e.g., the Rescher quantifier $\mathrm{More} = \{(D, P_1, P_2) \mid |P_1| > |P_2|\}$ of width $2$, expressing ``more neighbours satisfy~$\varphi$ than~$\psi$''. Second, we may allow them to involve arbitrary binary relations over the domain. This enables us to encode, e.g., the global modality using the universal relation $U = V \times V$.

Theorem \ref{thm:generalized-quantifiers-characterization} holds even when $\mathcal{Q}$ contains arbitrary finite-width quantifiers over arbitrary binary relations; the proof simply generalizes the game to use $n$-tuples of sets instead of individual sets, where $n$ is the quantifier width (see Appendix \ref{appendix:finite-width} for details).


\section{Experiments}
\label{sec:experiments}

\subsection{Experimental setting}

We evaluate $\mathcal{Q}$-WL-RF with the three quantifier sets from
Section~\ref{subsec:crude-wl}, using random forests~\cite{DBLP:journals/ml/Breiman01}
with default hyperparameters for tabular classification.
We compare to three other paradigms: the WL subtree kernel
(WL-GK)~\cite{DBLP:journals/jmlr/ShervashidzeSLMB11} with an SVM classifier;
graph convolutional networks (GCNs)~\cite{DBLP:conf/iclr/KipfW17} and graph
isomorphism networks (GINs)~\cite{xu2018powerful}; and the graph transformers
GraphGPS~\cite{rampavsek2022recipe} and Exphormer \cite{shirzad2023exphormer}. For a strong GNN baseline, we further enhance the standard GCN and GIN architectures following the GCN+ and
GIN+ recipe of \cite{luocan}, adding edge features, residual connections,
per-layer feed-forward networks with batch normalization, positional encodings and a sum-pooling
MLP head.

In addition to predictive performance, we measure wall-clock time and peak memory
per train/validation/test split.
All experiments were run on a single machine with a 24-core CPU, 128\,GB RAM, and
an NVIDIA RTX PRO 6000 GPU (96\,GB VRAM); implementation details and library
versions are listed in Appendix~\ref{appendix-hyperparameters}.


\subsection{Datasets}\label{datasets}

We use 14 graph classification datasets: 7 from TUDataset~\cite{tudatasetbench2020},
6 from the Open Graph Benchmark (OGB)~\cite{hu2020ogb}, and \texttt{Peptides-func}
from the Long Range Graph Benchmark (LRGB)~\cite{dwivedi2022LRGB}.
These span molecular, biological, and social network domains and include both binary
and multi-label tasks; all contain ${\geq}\,1000$ graphs.

For datasets with predefined splits (OGB, LRGB), we use them and repeat each experiment three times with different seeds.
For TU datasets, we use stratified 10-fold cross-validation.
Dataset statistics are in Appendix~\ref{appendix-hyperparameters}.

Our method and WL-GK operate on discrete node labels only: for those methods, we discard floating-point
node attributes (present in, e.g., \texttt{PROTEINS}) and do not use edge attributes.
GNN and GT baselines have access to both. For datasets without node labels, such as \texttt{IMDB-BINARY}, we follow usual convention and assign each node's out-degree as its initial label. For datasets with multiple classification tasks, such as \texttt{ogbg-molsider}, our methods and WL-GK train separate classifiers for each task, and GNN and GT methods train jointly on all tasks using binary cross-entropy loss.


\subsection{Hyperparameter optimization}

For all neural methods, we perform hyperparameter optimization using the
Tree-structured Parzen Estimator (TPE) implemented in
Optuna~\cite{optuna_2019}, subject to a wall-clock budget of 30~minutes
per method per dataset.
Each trial is evaluated on the validation split; the best configuration is
retrained for the full epoch budget on the complete training set.
The search covers architecture size (hidden dimension, number of layers),
regularization (dropout, weight decay), and optimization parameters (learning
rate, batch size); all neural methods use the AdamW optimizer with cosine
annealing preceded by a linear warmup.
Full search spaces and training details are given in
Appendix~\ref{appendix-hyperparameters}.

A benefit of our method is that hyperparameter tuning is minimal, since
  the only hyperparameters are the refinement depths $d_v \in \{0, \ldots, 10\}$,
  one per $\mathcal{Q}$-WL variant $v \in \mathcal{V}$.
  For $\mathcal{Q}$-WL-RF we select each $d_v$, and for WL-GK the single depth $d$,
  using a random 50\%/50\% train/validation split if predefined splits are not available,
  following \cite{DBLP:journals/jmlr/ShervashidzeSLMB11}. We stop increasing the number of rounds if the validation accuracy does not increase. $\mathcal{Q}$-WL-RF and WL-GK use completely unoptimized default hyperparameters for the random forest and SVM classifiers.


\subsection{Results}

\textbf{1. Our approach generally obtains results comparable to the other paradigms, and in some cases clearly better ones.} Raw results can be found in Table~\ref{tab:results}. The critical difference diagram in Figure \ref{fig:criticaldifference} shows that overall, our method outranks WL-GK and has no statistically significant difference to the GNNs or GTs. We also obtain some clear wins on multiple OGB datasets, as well as on \texttt{Peptides-func}. On the other hand, on \texttt{ogbg-molclintox}, our method vastly underperforms compared to neural methods.

\begin{figure}[t]
\caption{Critical difference diagram comparing method test performance \cite{demvsar2006statistical}. The six methods are listed by mean rank, lower being better. Methods with no statistically significant difference, as determined by the Nemenyi test with significance $\alpha=0.05$, are connected with a bar.}
\centering
\begin{tikzpicture}

  \tikzset{
    axis/.style={line width=0.8pt},
    tick/.style={line width=0.8pt},
    connector/.style={line width=0.6pt},
    sigbar/.style={line width=2.5pt},
    label/.style={font=\small},
    ranklabel/.style={font=\scriptsize, text=gray},
    methoddot/.style={fill=black, circle, inner sep=1.2pt},
  }

  \draw[axis] (2.5,0) -- (7.5,0);

  \draw[tick] (7.500,0) -- (7.500,0.25);
  \node[above] at (7.500,0.25) {1};
  \draw[tick] (6.500,0) -- (6.500,0.25);
  \node[above] at (6.500,0.25) {2};
  \draw[tick] (5.500,0) -- (5.500,0.25);
  \node[above] at (5.500,0.25) {3};
  \draw[tick] (4.500,0) -- (4.500,0.25);
  \node[above] at (4.500,0.25) {4};
  \draw[tick] (3.500,0) -- (3.500,0.25);
  \node[above] at (3.500,0.25) {5};
  \draw[tick] (2.500,0) -- (2.500,0.25);
  \node[above] at (2.500,0.25) {6};
  \draw[tick] (7.000,0) -- (7.000,0.15);
  \draw[tick] (6.000,0) -- (6.000,0.15);
  \draw[tick] (5.000,0) -- (5.000,0.15);
  \draw[tick] (4.000,0) -- (4.000,0.15);
  \draw[tick] (3.000,0) -- (3.000,0.15);

  \node[methoddot] at (6.321,0) {};
  \node[methoddot] at (5.571,0) {};
  \node[methoddot] at (4.7,0) {};
  \node[methoddot] at (4.8,0) {};
  \node[methoddot] at (4.500,0) {};
  \node[methoddot] at (4.143,0) {};

  \draw[connector] (4.143,0) -- (4.143,-0.300) -- (2.300,-0.300);
  \node[label, anchor=east] at (2.300,-0.300) {WL-GK {\scriptsize (4.4)}};
  \draw[connector] (4.500,0) -- (4.500,-0.700) -- (2.300,-0.700);
  \node[label, anchor=east] at (2.300,-0.700) {GCN+ {\scriptsize (4.0)}};
  \draw[connector] (4.7,0) -- (4.7,-1.100) -- (2.300,-1.100);
  \node[label, anchor=east] at (2.300,-1.100) {GPS {\scriptsize (3.8)}};
  \draw[connector] (6.321,0) -- (6.321,-0.300) -- (7.700,-0.300);
  \node[label, anchor=west] at (7.700,-0.300) {{\scriptsize (2.2)} $\mathcal{Q}$-WL-RF (Ours)};
  \draw[connector] (5.571,0) -- (5.571,-0.700) -- (7.700,-0.700);
  \node[label, anchor=west] at (7.700,-0.700) {{\scriptsize (2.9)} GIN+};
  \draw[connector] (4.80,0) -- (4.80,-1.100) -- (7.700,-1.100);
  \node[label, anchor=west] at (7.700,-1.100) {{\scriptsize (3.8)} Exph.};

  \draw[sigbar] (4.450,-0.150) -- (6.371,-0.150);
  \draw[sigbar] (4.093,-0.270) -- (5.621,-0.270);

\end{tikzpicture}
\label{fig:criticaldifference}
\end{figure}

\begin{table}[t]
\centering
\caption{Test performance across datasets (mean $\pm$ std over 10-fold CV or the 3 seeds). \underline{\textcolor{red}{\textbf{Red}}}: best. \textbf{Bold}: within 1 std of best. \textcolor{gray}{Gray}: more than 1 std below best.}
\label{tab:results}
\setlength{\tabcolsep}{2pt}
{%
\begin{tabular}{l c c cc cc}
\toprule
& \textbf{$\mathcal{Q}$-WL-RF} & \textbf{Kernel} & \multicolumn{2}{c}{\textbf{GNNs}} & \multicolumn{2}{c}{\textbf{Graph Transformers}} \\
\cmidrule(lr){2-2} \cmidrule(lr){3-3} \cmidrule(lr){4-5} \cmidrule(lr){6-7}
\textbf{Dataset} & (Ours) & WL & GIN+ & GCN+ & GraphGPS & Exphormer \\
\midrule
\multicolumn{7}{l}{\textit{TU -- Small Molecules (Accuracy $\uparrow$)}} \\
Mutagen. & \textbf{.840\tiny$\pm$.020} & \textbf{.830\tiny$\pm$.019} & \underline{\textbf{\textcolor{red}{.841\tiny$\pm$.021}}} & \textcolor{gray}{.817\tiny$\pm$.021} & \textbf{.830\tiny$\pm$.014} & \textbf{.828\tiny$\pm$.014} \\
NCI1 & \underline{\textbf{\textcolor{red}{.857\tiny$\pm$.015}}} & \textbf{.855\tiny$\pm$.015} & \textbf{.836\tiny$\pm$.024} & \textcolor{gray}{.791\tiny$\pm$.018} & \textcolor{gray}{.825\tiny$\pm$.015} & \textcolor{gray}{.827\tiny$\pm$.014} \\
NCI109 & \textbf{.848\tiny$\pm$.013} & \underline{\textbf{\textcolor{red}{.857\tiny$\pm$.017}}} & \textcolor{gray}{.838\tiny$\pm$.013} & \textcolor{gray}{.764\tiny$\pm$.020} & \textcolor{gray}{.803\tiny$\pm$.018} & \textcolor{gray}{.817\tiny$\pm$.012} \\
\midrule
\multicolumn{7}{l}{\textit{TU -- Bioinformatics (Accuracy $\uparrow$)}} \\
DD & \underline{\textbf{\textcolor{red}{.796\tiny$\pm$.037}}} & \textbf{.789\tiny$\pm$.034} & \textbf{.772\tiny$\pm$.044} & \textbf{.746\tiny$\pm$.063} & \textbf{.767\tiny$\pm$.036} & \textcolor{gray}{.750\tiny$\pm$.030} \\
PROTEINS & \underline{\textbf{\textcolor{red}{.755\tiny$\pm$.035}}} & \textbf{.725\tiny$\pm$.037} & \textbf{.734\tiny$\pm$.052} & \textcolor{gray}{.700\tiny$\pm$.049} & \textbf{.745\tiny$\pm$.029} & \textbf{.694\tiny$\pm$.095} \\
\midrule
\multicolumn{7}{l}{\textit{TU -- Social Networks (Accuracy $\uparrow$)}} \\
IMDB-B & \textbf{.727\tiny$\pm$.034} & \textbf{.713\tiny$\pm$.035} & \textbf{.735\tiny$\pm$.014} & \underline{\textbf{\textcolor{red}{.743\tiny$\pm$.026}}} & \textbf{.691\tiny$\pm$.057} & \textbf{.738\tiny$\pm$.028} \\
IMDB-M & \textbf{.500\tiny$\pm$.039} & \textbf{.500\tiny$\pm$.037} & \underline{\textbf{\textcolor{red}{.503\tiny$\pm$.029}}} & \textbf{.499\tiny$\pm$.034} & \textcolor{gray}{.481\tiny$\pm$.028} & \textbf{.502\tiny$\pm$.022} \\
\midrule
\multicolumn{7}{l}{\textit{OGB Molecular (ROC-AUC $\uparrow$)}} \\
moltox21 & \textcolor{gray}{.745\tiny$\pm$.004} & \textcolor{gray}{.727\tiny$\pm$.000} & \textcolor{gray}{.758\tiny$\pm$.005} & \underline{\textbf{\textcolor{red}{.776\tiny$\pm$.010}}} & \textcolor{gray}{.758\tiny$\pm$.007} & \textbf{.773\tiny$\pm$.006} \\
molbace & \underline{\textbf{\textcolor{red}{.834\tiny$\pm$.010}}} & \textcolor{gray}{.758\tiny$\pm$.000} & \textcolor{gray}{.773\tiny$\pm$.002} & \textcolor{gray}{.785\tiny$\pm$.018} & \textcolor{gray}{.795\tiny$\pm$.025} & \textcolor{gray}{.773\tiny$\pm$.034} \\
molbbbp & \underline{\textbf{\textcolor{red}{.734\tiny$\pm$.005}}} & \textcolor{gray}{.653\tiny$\pm$.000} & \textcolor{gray}{.686\tiny$\pm$.006} & \textcolor{gray}{.665\tiny$\pm$.007} & \textcolor{gray}{.664\tiny$\pm$.008} & \textcolor{gray}{.694\tiny$\pm$.012} \\
molclintox & \textcolor{gray}{.720\tiny$\pm$.054} & \textcolor{gray}{.537\tiny$\pm$.041} & \textcolor{gray}{.869\tiny$\pm$.018} & \textcolor{gray}{.852\tiny$\pm$.036} & \underline{\textbf{\textcolor{red}{.901\tiny$\pm$.034}}} & \textbf{.889\tiny$\pm$.020} \\
molsider & \underline{\textbf{\textcolor{red}{.658\tiny$\pm$.004}}} & \textcolor{gray}{.540\tiny$\pm$.031} & \textcolor{gray}{.611\tiny$\pm$.002} & \textcolor{gray}{.605\tiny$\pm$.003} & \textcolor{gray}{.599\tiny$\pm$.017} & \textcolor{gray}{.587\tiny$\pm$.008} \\
molhiv & \textbf{.767\tiny$\pm$.013} & \textcolor{gray}{OOM} & \textcolor{gray}{.747\tiny$\pm$.018} & \textcolor{gray}{.763\tiny$\pm$.012} & \underline{\textbf{\textcolor{red}{.775\tiny$\pm$.012}}} & \textbf{.764\tiny$\pm$.017} \\
\midrule
\multicolumn{7}{l}{\textit{Long Range Graph Benchmark (AP $\uparrow$)}} \\
Pep.-func & \underline{\textbf{\textcolor{red}{.710\tiny$\pm$.001}}} & \textcolor{gray}{.671\tiny$\pm$.000} & \textcolor{gray}{.676\tiny$\pm$.004} & \textcolor{gray}{.677\tiny$\pm$.000} & \textcolor{gray}{.652\tiny$\pm$.012} & \textcolor{gray}{.564\tiny$\pm$.008} \\
\bottomrule
\end{tabular}%
}
\end{table}

\textbf{2. Our approach is much faster than neural methods and more memory-efficient than graph kernels.} As seen in Figure \ref{fig:runtime}, GNNs are roughly 5$\times$ and GTs 11--17$\times$ slower to train than $\mathcal{Q}$-WL-RF. When taking hyperparameter optimization into account, the difference blows up to 60--71$\times$ and 82--85$\times$, respectively. For WL-GK, the similarity matrix computation scales poorly as the number of graphs in the dataset grows, which is clearly seen in Figure \ref{fig:memory}.

\begin{figure}[t]
  \centering
  \begin{minipage}[t]{0.48\columnwidth}
  \centering
  \caption{Mean runtime per train/val/test split, averaged across all datasets. $\mathcal{Q}$-WL-RF and WL-GK include time used for HPO.
  Slowdown factors are relative to $\mathcal{Q}$-WL-RF\@.}
  \begin{tikzpicture}
  \begin{axis}[
      xbar,
      width=\columnwidth,
      height=6.4cm,
      xlabel={Mean runtime per split (s, log scale)},
      xlabel style={font=\scriptsize},
      xmode=log,
      log origin=infty,
      xmin=10, xmax=8000,
      ytick={1,2,4,6,8,10},
      yticklabels={%
          $\mathcal{Q}$-WL-RF,
          WL-GK,
          GIN+ \scriptsize{},
          GCN+ \scriptsize{},
          GraphGPS \scriptsize{},
          Exphormer \scriptsize{},
      },
      yticklabel style={font=\small},
      ytick style={draw=none},
      y dir=reverse,
      bar width=7pt,
      enlarge y limits={abs=0.6cm},
      tick label style={font=\small},
      axis x line*=bottom,
      axis y line*=left,
      xmajorgrids=true,
      grid style={gray!25, thin},
      xtick={10, 100, 1000},
      xticklabels={10, 100, 1000},
      clip=false,
      extra x ticks={27.3},
      extra x tick labels={},
      extra x tick style={grid style={blue!70!black, dashed, thin}},
  ]

  \addplot[
      fill=blue!55!cyan!70, draw=blue!70!black, fill opacity=0.85
  ] coordinates {
      (27.3, 0)
  };

  \addplot[
      fill=red!15, draw=red!50!black, fill opacity=0.85
  ] coordinates {
      (143.6, 3.5)
      (127.1, 5.5)
      (295.1, 7.5)
      (460.1, 9.5)
  };

  \addplot[
      fill=red!55, draw=red!50!black, fill opacity=0.95
  ] coordinates {
      (33.2, 3)
      (1638, 5.25)
      (1941, 7.25)
      (1933, 9.25)
      (2316, 11.25)
  };

  \node[font=\scriptsize, text=red!60!black, anchor=west]
      at (axis cs:33.2, 2) {\;\;$1.2\times$};
  \node[font=\scriptsize, text=red!60!black, anchor=west]
      at (axis cs:143.6, 3.5) {\;\;$5.3\times$};
  \node[font=\scriptsize\bfseries, text=red!60!black, anchor=west]
      at (axis cs:1638, 4.5) {\;\;$60\times$};
  \node[font=\scriptsize, text=red!60!black, anchor=west]
      at (axis cs:127.1, 5.5) {\;\;$4.7\times$};
  \node[font=\scriptsize\bfseries, text=red!60!black, anchor=west]
      at (axis cs:1941, 6.5) {\;\;$71\times$};
  \node[font=\scriptsize, text=red!60!black, anchor=west]
      at (axis cs:295, 7.5) {\;\;$11\times$};
  \node[font=\scriptsize\bfseries, text=red!60!black, anchor=west]
      at (axis cs:1933, 8.5) {\;\;$82\times$};
  \node[font=\scriptsize, text=red!60!black, anchor=west]
      at (axis cs:460.1, 9.5) {\;\;$17\times$};
  \node[font=\scriptsize\bfseries, text=red!60!black, anchor=west]
      at (axis cs:2316, 10.5) {\;\;$85\times$};

  \node[font=\scriptsize, fill=blue!55!cyan!70, fill opacity=0.85,
        draw=blue!70!black, minimum width=8pt, minimum height=6pt,
        inner sep=0pt] at (axis cs:300, 1) {};
  \node[font=\scriptsize, anchor=west] at (axis cs:360, 1) {Ours};
  \node[font=\scriptsize, fill=red!15, fill opacity=0.85,
        draw=red!50!black, minimum width=8pt, minimum height=6pt,
        inner sep=0pt] at (axis cs:1500, 1) {};
  \node[font=\scriptsize, anchor=west] at (axis cs:1800, 1) {No HPO};
  \node[font=\scriptsize, fill=red!55, fill opacity=0.95,
        draw=red!50!black, minimum width=8pt, minimum height=6pt,
        inner sep=0pt] at (axis cs:300, 2) {};
  \node[font=\scriptsize, anchor=west] at (axis cs:360, 2) {HPO};

  \end{axis}
  \end{tikzpicture}
  \label{fig:runtime}
  \end{minipage}
\hfill
\begin{minipage}[t]{0.48\columnwidth}
\centering
\caption{Memory consumption of $\mathcal{Q}$-WL-RF and WL-GK as a function of dataset size. Unlike the graph kernel, our method remains efficient even on larger datasets.}
\begin{tikzpicture}
\begin{axis}[
    width=\textwidth,
    height=4.8cm,
    xlabel={Number of graphs},
    y coord trafo/.code={\pgfmathparse{#1/1000}\pgfmathresult},
    ylabel={Memory (GB)},
    xlabel style={font=\scriptsize},
    ylabel style={font=\scriptsize, at={(axis description cs:-0.1,0.5)}},
    tick label style={font=\tiny},
    legend style={
        font=\tiny,
        at={(0.02,0.98)},
        anchor=north west,
    },
    xmode=log,
    grid=major,
    grid style={line width=0.2pt, draw=gray!30},
    legend cell align=left,
]
\addplot[
    color=blue,
    mark=square*,
    only marks,
    mark size=1.2pt,
] coordinates {
    (1000, 909.9)  
  (1113, 938.6)  
  (1178, 1691.4) 
  (1427, 903.8)  
  (1477, 893.4)  
  (1500, 908.0)  
  (1513, 974.1)  
  (2039, 924.7)  
  (4110, 1058.1) 
  (4127, 1062.8) 
  (4337, 1014.2) 
  (7831, 964.4)  
  (15535, 1307.7) 
  (41127, 1832.9) 
};
\addlegendentry{$\mathcal{Q}$-WL-RF}
\addplot[
    color=orange,
    mark=*,
    only marks,
    mark size=1.2pt,
] coordinates {
    (1000, 1112.0) (1113, 1123.3) (1178, 1225.1) (1427, 986.2)
    (1477, 974.4) (1500, 1103.0) (1513, 975.8) (2039, 1037.5)
    (4110, 2511.8) (4127, 2406.5) (4337, 1871.7) (7831, 3360.8)
};
\addlegendentry{WL-GK}
\draw[->, densely dotted, orange, thick] (axis cs:15535,3100) -- (axis cs:15535,3600);
\node[anchor=north, font=\tiny\bfseries, orange] at (axis cs:15535,3100) {33 GB};
\draw[->, densely dotted, orange, thick] (axis cs:41127,3100) -- (axis cs:41127,3600);
\node[anchor=north, font=\tiny\bfseries, orange] at (axis cs:41127,3100) {$>$128 GB};
\end{axis}
\end{tikzpicture}
\label{fig:memory}
\end{minipage}
\end{figure}

\textbf{3. Full WL is good on its own, but cruder modalities help.} Table \ref{tab:qwlu-ablation} shows an ablation study of the modality sets of $\mathcal{Q}$-WL-RF. We observe that Full is the most crucial modality, but Plain and Majority are still useful on average, with them being most impactful on the \texttt{molclintox} dataset; interestingly, that dataset is the one where our method performs the worst.

\begin{table}[t]
  \centering
  \caption{Drop-one ablation for $\mathcal{Q}$-WL-RF, summarized across all 14 datasets.}
  \label{tab:qwlu-ablation}
  \small
  \begin{tabular}{lccc}
  \toprule
  \textbf{Variant removed} & \textbf{Mean $\Delta$} & \textbf{Weakens performance} & \textbf{Worst case} \\
  \midrule
  Full        & $-0.037$ & 14/14 & $-0.404$ (\texttt{Peptides-func}) \\
  Plain       & $-0.004$ & 9/14  & $-0.046$ (\texttt{molclintox}) \\
  Majority    & $-0.002$ & 10/14 & $-0.026$ (\texttt{molclintox}) \\
  \bottomrule
  \end{tabular}
\end{table}

\textbf{4. The optimal number of rounds of ($\mathcal{Q}$-)WL is usually low.} The number usually ranges from 0--3, with the mean value being 1.39 for Full WL, 1.50 for Plain WL, 1.59 for Majority WL and 2.56 for WL-GK; more details in Appendix \ref{appendix-hyperparameters}. This could imply that the datasets in question have little useful geometric information beyond 1- or 2-hop neighbourhoods.

\section{Conclusion}
\label{sec:conclusion}

We introduced a novel approach to graph classification by transforming graph data into tabular form via variants of the WL-algorithm. We established a general theoretical result characterizing the expressive power of our classifiers and evaluated our approach on 14 benchmark datasets, demonstrating competitive results to graph kernels, GNNs and GTs while being computationally much more efficient.

Beyond its use as a standalone classifier, we believe our method has practical
value as a rapid diagnostic tool for graph datasets: because a single run
takes seconds rather than hours, practitioners can quickly test how
incorporating different data affects classification, and use the optimal depth~$d$ as a
probe for whether graph structure contributes useful signal at all; if
$d = 0$ is optimal, the topology adds nothing beyond the node labels
themselves.

\paragraph{Limitations.} \textbf{1. Datasets.} Due to hardware and time limitations, we were unable to perform tests on very large graph classification datasets, such as \texttt{ogbg-molpcba} and \texttt{ogbg-ppa}. Our method does perform competitively on \texttt{ogbg-molhiv} and \texttt{Peptides-func}, the largest datasets in our collection, but further work would be needed to confirm if this continues to scale to datasets with hundreds of thousands or millions of graphs.

\textbf{2. Tasks.} Experiments concern only graph classification, while ignoring graph regression, link prediction and all node-level tasks. $\mathcal{Q}$-WL-preprocessing could easily be a part of all these, and node classification could plausibly be implemented by treating each node as a separate graph in the tabularization phase, but how to extend our method to the other domains remains unclear.

\textbf{3. Discrete node labels.} Neither edge information nor continuous node attributes were used by $\mathcal{Q}$-WL-RF. Preliminary experiments using, e.g., $k$-means clustering to discretize and concatenate this data to the node labels resulted in inconclusive changes in accuracy with a remarkably increased runtime (and increased complexity of the pipeline). Interestingly, our method obtains comparable accuracies to GNNs and GTs despite the fact that they do utilize edge information and continuous node attributes.

Multiple directions remain open for future work. First, the $\mathcal{Q}$-WL variants we test correspond to relatively simple weakenings of graded modal logic; more expressive quantifiers (e.g., those detecting local structures like triangles) remain unexplored. Second, further research on the tabularization step (e.g., systematically comparing different tabular classifiers and analyzing feature importances) is left for future work.

\newpage

\bibliographystyle{plainurl}
\bibliography{arxiv_new}

\clearpage
\appendix

\section{Supplementary Experimental Details}\label{appendix-hyperparameters}

\textbf{Implementation.}
  All code is written in Python 3.12.3.
  $\mathcal{Q}$-WL-RF uses the random forest implementation from scikit-learn
  1.7.1~\cite{scikit-learn}, and WL-GK uses the WL subtree kernel implementation
  from GraKel 0.1.8~\cite{JMLR:v21:18-370} with scikit-learn's SVM classifier.
  GCN, GIN, and Exphormer are implemented using PyTorch Geometric
  2.6.1~\cite{Fey/Lenssen/2019} on top of PyTorch 2.9.1, and GraphGPS uses the
  original codebase of~\cite{rampavsek2022recipe}.
  All experiments were run on a Windows 11 desktop with an Intel Core Ultra 9 285
  24-core processor, 128\,GB of RAM, and an NVIDIA RTX PRO 6000 Blackwell Max-Q
  Workstation Edition GPU with 96\,GB of VRAM.
  $\mathcal{Q}$-WL-RF and WL-GK were run on the CPU; GNNs (GCN, GIN), Exphormer,
  and GraphGPS were run on the GPU.

  \paragraph{Hyperparameters.} For all neural baselines (GIN+, GCN+, Exphormer, GraphGPS) we use Optuna
  4.8.0~\cite{optuna_2019} with the TPE sampler under a fixed wall-clock budget
  of 30~minutes per (method, dataset) cell.
  The search spaces are listed in
  Tables~\ref{tab:gnn-hpo}, \ref{tab:exphormer-hpo}, and~\ref{tab:gps-hpo}.

  The number of training epochs is \emph{not} part of the Optuna search.
  Instead, every training run both inside an HPO trial and at the final
  retraining stage trains for up to $300$ epochs while tracking
  validation performance after each epoch, and the parameters from the
  best-validation epoch are used for evaluation. Training also early-stops if the validation accuracy reaches $100\%$. For the TUDataset 10-fold CV, to avoid prohibitive time cost, we use $50$ epochs inside the HPO trials and $300$ epochs for the final evaluation. 
  
  GIN and GCN use AdamW with a linear-warmup-then-cosine-anneal learning rate
  schedule and gradient clipping at norm $1.0$; Exphormer uses the same
  optimiser and schedule.
  GraphGPS uses its original codebase~\cite{rampavsek2022recipe} configuration
  system, with HPO supplying CLI overrides for the searched parameters.

  \begin{table}[p]
  \centering
  \caption{HPO search space for GIN+ and GCN+.}
  \label{tab:gnn-hpo}
  \small
  \begin{tabular}{ll}
  \toprule
  \textbf{Hyperparameter} & \textbf{Range / values} \\
  \midrule
  Hidden dimension       & $\{32, 64, 128, 256\}$ \\
  Number of layers       & integer $\in [2, 12]$ \\
  Dropout                & $[0.0, 0.3]$, step $0.05$ \\
  Learning rate          & log-uniform $[10^{-4}, 10^{-2}]$ \\
  Weight decay           & log-uniform $[10^{-6}, 10^{-3}]$ \\
  Batch size             & $\{32, 64, 128\}$ \\
  RWSE positional dim.   & $\{0, 16, 20\}$ \\
  \bottomrule
  \end{tabular}
  \end{table}

  \begin{table}[p]
  \centering
  \caption{HPO search space for Exphormer.}
  \label{tab:exphormer-hpo}
  \small
  \begin{tabular}{ll}
  \toprule
  \textbf{Hyperparameter} & \textbf{Range / values} \\
  \midrule
  Hidden dimension       & $\{32, 64, 128\}$ (must be divisible by heads) \\
  Number of layers       & integer $\in [2, 6]$ \\
  Number of heads        & $\{2, 4, 8\}$ \\
  Dropout                & $[0.0, 0.3]$, step $0.1$ \\
  Attention dropout      & $[0.0, 0.5]$, step $0.1$ \\
  Learning rate          & log-uniform $[10^{-4}, 10^{-2}]$ \\
  Weight decay           & log-uniform $[10^{-6}, 10^{-3}]$ \\
  Batch size             & $\{32, 64, 128\}$ \\
  Warmup epochs          & $\{5, 10, 20\}$ \\
  Expander degree        & $\{2, 3, 4\}$ \\
  RWSE positional dim.   & fixed at $16$ \\
  \bottomrule
  \end{tabular}
  \end{table}

  \begin{table}[p]
  \centering
  \caption{HPO search space for GraphGPS.}
  \label{tab:gps-hpo}
  \small
  \begin{tabular}{ll}
  \toprule
  \textbf{Hyperparameter} & \textbf{Range / values} \\
  \midrule
  Hidden dimension       & $\{32, 64, 128, 256\}$ (must be divisible by heads) \\
  Number of layers       & integer $\in [2, 8]$ \\
  Number of heads        & $\{2, 4, 8\}$ \\
  Dropout                & $[0.0, 0.3]$, step $0.1$ \\
  Attention dropout      & $[0.0, 0.5]$, step $0.1$ \\
  Learning rate          & log-uniform $[10^{-4}, 10^{-2}]$ \\
  Weight decay           & log-uniform $[10^{-6}, 10^{-3}]$ \\
  Batch size             & $\{16, 32, 64, 128\}$ \\
  \bottomrule
  \end{tabular}
  \end{table}

$\mathcal{Q}$-WL-RF uses the random forest classifier with the default hyperparameters of the scikit-learn 1.7.1 implementation, listed in Table \ref{tab:rf-hyperparams}, and WL-GK uses the SVM classifier with the default hyperparameters of the scikit-learn 1.7.1 implementation, listed in Table \ref{tab:svm-hyperparams}. The number of optimal rounds of ($\mathcal{Q}$)-WL per dataset used for WL-GK and for $\mathcal{Q}$-WL-RF is listed in Table \ref{tab:depths}.

\begin{table}[p]
\centering
\caption{Default random forest hyperparameter values (scikit-learn 1.7.1).}
\label{tab:rf-hyperparams}
\begin{tabular}{ll}
\toprule
\textbf{Hyperparameter} & \textbf{Value} \\
\midrule
\texttt{max\_depth} & None \\
\texttt{n\_estimators} & 100 \\
\texttt{criterion} & gini \\
\texttt{max\_features} & sqrt \\
\texttt{min\_samples\_split} & 2 \\
\texttt{min\_samples\_leaf} & 1 \\
\texttt{bootstrap} & True \\
\texttt{min\_impurity\_decrease} & 0.0 \\
\bottomrule
\end{tabular}
\end{table}

\begin{table}[p]
\centering
\caption{Default SVM hyperparameter values (scikit-learn 1.7.1).}
\label{tab:svm-hyperparams}
\begin{tabular}{ll}
\toprule
\textbf{Hyperparameter} & \textbf{Value} \\
\midrule
\texttt{C} & 1.0 \\
\texttt{shrinking} & True \\
\texttt{probability} & False \\
\texttt{tol} & 0.001 \\
\texttt{cache\_size} & 200 \\
\texttt{class\_weight} & None \\
\texttt{max\_iter} & $-1$ (no limit) \\
\texttt{decision\_function\_shape} & ovr \\
\texttt{break\_ties} & False \\
\bottomrule
\end{tabular}
\end{table}

\paragraph{Datasets.} Statistics of the datasets are listed in Table \ref{tab:datasets}.

\begin{table}[p]
\centering
\caption{Comparing the optimal number of rounds of ($\mathcal{Q}$)-WL across selected methods and datasets.}
\label{tab:depths}
\small
\setlength{\tabcolsep}{4pt}
\begin{tabular}{lcccc}
\toprule
\textbf{Dataset} & Full & Plain & Majority & WL-GK \\
\midrule
Mutagenicity & 1.43 {\scriptsize ±0.46} & 1.96 {\scriptsize ±0.34} & 1.56 {\scriptsize ±0.52} & 2.78 {\scriptsize ±1.21} \\
NCI1 & 2.12 {\scriptsize ±0.23} & 3.14 {\scriptsize ±0.52} & 2.72 {\scriptsize ±0.46} & 6.28 {\scriptsize ±1.44} \\
NCI109 & 2.23 {\scriptsize ±0.40} & 3.44 {\scriptsize ±0.56} & 2.80 {\scriptsize ±0.40} & 5.52 {\scriptsize ±0.90} \\
DD & 0.00 {\scriptsize ±0.00} & 0.00 {\scriptsize ±0.00} & 0.00 {\scriptsize ±0.00} & 0.92 {\scriptsize ±0.78} \\
PROTEINS & 0.88 {\scriptsize ±0.27} & 0.96 {\scriptsize ±0.67} & 0.72 {\scriptsize ±0.46} & 0.56 {\scriptsize ±0.92} \\
IMDB-B & 0.00 {\scriptsize ±0.00} & 0.00 {\scriptsize ±0.00} & 0.00 {\scriptsize ±0.00} & 0.96 {\scriptsize ±1.11} \\
IMDB-M & 0.03 {\scriptsize ±0.10} & 0.16 {\scriptsize ±0.31} & 0.06 {\scriptsize ±0.14} & 0.62 {\scriptsize ±0.87} \\
moltox21 & 1.00 {\scriptsize ±0.00} & 2.40 {\scriptsize ±0.49} & 2.20 {\scriptsize ±0.40} & 2.00 {\scriptsize ±0.00} \\
molbace & 1.00 {\scriptsize ±0.00} & 3.40 {\scriptsize ±0.49} & 3.60 {\scriptsize ±1.74} & 1.00 {\scriptsize ±0.00} \\
molbbbp & 2.00 {\scriptsize ±0.00} & 1.40 {\scriptsize ±0.49} & 1.20 {\scriptsize ±0.40} & 1.00 {\scriptsize ±0.00} \\
molclintox & 1.00 {\scriptsize ±0.00} & 1.20 {\scriptsize ±0.40} & 1.00 {\scriptsize ±0.00} & 0.60 {\scriptsize ±0.49} \\
molsider & 2.00 {\scriptsize ±0.00} & 0.40 {\scriptsize ±0.80} & 2.00 {\scriptsize ±0.00} & 1.00 {\scriptsize ±2.00} \\
molhiv & 1.60 {\scriptsize ±0.49} & 1.20 {\scriptsize ±0.40} & 2.60 {\scriptsize ±0.49} & \textcolor{gray}{OOM} \\
Peptides-func & 4.20 {\scriptsize ±0.40} & 1.40 {\scriptsize ±0.49} & 1.40 {\scriptsize ±0.49} & 10.00 {\scriptsize ±0.00} \\
\midrule
\textbf{Mean} & \textbf{1.39} & \textbf{1.50} & \textbf{1.56} & \textbf{2.56} \\
\bottomrule
\end{tabular}
\end{table}

\begin{table}[p]
\centering
\caption{Dataset statistics.}
\label{tab:datasets}
\footnotesize
\setlength{\tabcolsep}{3pt}
\begin{tabular}{lrrrrr}
\toprule
\textbf{Dataset} & \textbf{Graphs} & \textbf{Avg. $|V|$} & \textbf{Avg. deg.} & \textbf{Classes} & \textbf{Tasks} \\
\midrule
Mutagenicity & 4,337  & 30.3  & 2.0  & 2  & 1\\
NCI1         & 4,110  & 29.9  & 2.2  & 2  & 1\\
NCI109       & 4,127  & 29.7  & 2.2  & 2  & 1\\
DD           & 1,178  & 284.3 & 5.0  & 2  & 1\\
PROTEINS     & 1,113  & 39.1  & 3.7  & 2  & 1\\
IMDB-B       & 1,000  & 19.8  & 9.8  & 2  & 1\\
IMDB-M       & 1,500  & 13.0  & 10.1 & 3  & 1\\
moltox21     & 7,831  & 18.6  & 2.1  & 2  & 12\\
molbace      & 1,513  & 34.1  & 2.2  & 2  & 1 \\
molbbbp      & 2,039  & 24.1  & 2.2  & 2  & 1 \\
molclintox   & 1,477  & 26.2  & 2.1  & 2  & 2 \\
molsider     & 1,427  & 33.6  & 2.1  & 2  & 27\\
molhiv       & 41,127 & 25.5  & 2.2  & 2  & 1 \\
Peptides-func & 15,535 & 150.9 & 2.0 & 2  & 10\\
\bottomrule
\end{tabular}
\end{table}

\section{Proof of Theorem \ref{thm:generalized-quantifiers-characterization}}\label{generalized-quantifiers-proof}

In this section, we provide the full proof for Theorem \ref{thm:generalized-quantifiers-characterization}. We will start with auxiliary definitions.

The definition of the $\mathcal{Q}$-bisimulation game can be found in Section \ref{sec:arbitrary-quantifiers}. In the $d$\textbf{-round} $\mathcal{Q}$\textbf{-bisimulation game}, Player 2 is declared the winner if Player 1 has not won after $d$ rounds. If Player 2 has a winning strategy for the $\mathcal{Q}$-bisimulation game (resp. $d$-round $\mathcal{Q}$-bisimulation game) on the pointed Kripke models $(\mathfrak{M},w)$ and $(\mathfrak{N},v)$, we write $\mathfrak{M},w\sim_{\mathcal{Q}}\mathfrak{N},v$ (resp. $\mathfrak{M},w\sim^d_{\mathcal{Q}}\mathfrak{N},v$).

We now define equivalence in $\mathrm{PL}(\mathcal{Q})$. Let $\mathcal{Q}$ be a set of generalized quantifiers. Recall that PL($\mathcal{Q}$) denotes propositional logic extended with the generalized modalities obtained from the quantifiers $Q\in\mathcal{Q}$ and PL(\(\mathcal{Q}\))\(^d\) denotes the set of formulas of PL(\(\mathcal{Q}\)) restricted to a maximum modal depth of $d$. If two pointed Kripke models $(\mathfrak{M},w)$ and $(\mathfrak{N},v)$ are indistinguishable by PL(\(\mathcal{Q}\))\(^d\), we write $\mathfrak{M},w\equiv_{\text{PL}(\mathcal{Q})}^d \mathfrak{N},v$.

Finally, we define \textbf{$\mathcal{Q}$-modal $d$-types}. Fix a Kripke model \(\mathfrak{M}\). For each \(w \in \mathfrak{M}\) we define \(\varphi_{\mathfrak{M},w}^{\mathcal{Q},d}\) recursively as follows. First, we define
\begin{align*}
    \varphi_{\mathfrak{M},w}^{\mathcal{Q},0} := \bigwedge \{p \mid \mathfrak{M},w \Vdash p\} \ \land \ \bigwedge\{\neg p \mid \mathfrak{M},w \Vdash \neg p\}.
\end{align*}
Then, having defined $\varphi^{\mathcal{Q},d}_{\mathfrak{M},w}$, we define
\begin{align*}
    & \varphi_{\mathfrak{M},w}^{\mathcal{Q},d+1} := \varphi_{\mathfrak{M},w}^{\mathcal{Q},0} \\
    & \land \\
    & \bigwedge_{\begin{array}{c}
       Q \in \mathcal{Q},\, C \subseteq \{\varphi_{\mathfrak{M},w}^{\mathcal{Q},d} \mid w \in \mathfrak{M}\} \\
        \mathfrak{M},w \Vdash \langle Q \rangle \bigvee\limits_{\varphi \in C} \varphi
    \end{array}} \langle Q \rangle \bigvee_{\varphi \in C} \varphi \\
    & \land \\
    & \bigwedge_{\begin{array}{c}
        Q \in \mathcal{Q},\, C \subseteq \{\varphi_{\mathfrak{M},w}^{\mathcal{Q},d} \mid w \in \mathfrak{M}\} \\
        \mathfrak{M},w \Vdash \neg \langle Q \rangle \bigvee\limits_{\varphi \in C} \varphi
    \end{array}} \neg \langle Q \rangle \bigvee_{\varphi \in C} \varphi.
\end{align*}

It is not difficult to see that, when given \(\mathfrak{M}\) as input, at round \(d\) the \(\mathcal{Q}\)-WL essentially computes \(\varphi_{\mathfrak{M},w}^{\mathcal{Q},d}\) for each node \(w \in \mathfrak{M}\). Note that if \(\mathcal{Q}\) is finite, \(\varphi_{\mathfrak{M},w}^{\mathcal{Q},d}\) is a formula of PL(\(\mathcal{Q}\))\(^d\). In the case where \(\mathcal{Q}\) is infinite this is no longer the case, because then \(\varphi_{\mathfrak{M},w}^{\mathcal{Q},d}\) contains infinitary conjunctions. (See Appendix \ref{appendix:effectively-finite} for how to make \(\mathcal{Q}\)-WL finitary when \(\mathcal{Q}\) is infinite but has certain computational properties.)

The rest of this section is devoted to proving the following reformulation of Theorem \ref{thm:generalized-quantifiers-characterization}: for every pair $(\mathfrak{M},w)$, $(\mathfrak{N},v)$ of pointed Kripke models and set $\mathcal{Q}$ of generalized modalities the following are equivalent:
\begin{enumerate}
    \item $\mathfrak{M}, w \sim^d_{\mathcal{Q}} \mathfrak{N},v$,
    \item $\mathfrak{M}, w \equiv^d_{\mathrm{PL}(\mathcal{Q})}\mathfrak{N},v$,
    \item $\mathfrak{N},v\Vdash \varphi^{\mathcal{Q},d}_{\mathfrak{M} \sqcup \mathfrak{N}, w}$.
\end{enumerate}
Here \(\mathfrak{M} \sqcup \mathfrak{N}\) denotes the disjoint union of \(\mathfrak{M}\) and \(\mathfrak{N}\).

We will prove the equivalence of \(1.,2.\) and \(3.\) via induction on \(d\). The base case \(d = 0\) is clear, because \(1.,2.\) and \(3.\) are just different ways of saying that \(\mathfrak{M},w\) and \(\mathfrak{N},v\) satisfy the same proposition symbols. For the induction step, assume that \(1.,2.\) and \(3.\) are equivalent for \(d\). We will prove their equivalence for \(d + 1\) via three Lemmas.

\begin{lemma}\label{bisimulation-iff-equiv}
    $\mathfrak{M},w\sim^{d+1}_\mathcal{Q} \mathfrak{N},v \implies \mathfrak{M},w\equiv_{\mathrm{PL}(\mathcal{Q})}^{d+1} \mathfrak{N},v$.
\end{lemma}
\begin{proof}
    
    Let $(\mathfrak{M},w)$ and $(\mathfrak{N},v)$ be pointed Kripke models and, contrapositively, assume that $\mathfrak{M},w\not\equiv_{\text{PL}(\mathcal{Q})}^{d+1}\mathfrak{N},v$. Without loss of generality, we may assume that $\mathfrak{M},w\Vdash\GQ \psi$ and $\mathfrak{N},v\Vdash\neg\GQ\psi$ for some $Q\in \mathcal{Q}$ and $\psi\in$ PL(\(\mathcal{Q}\))\(^d\).
    
    Player 1 now begins the game by choosing
    \[X := \{w' \in N(w) \mid \mathfrak{M},w' \Vdash \psi\}.\]
    If Player 2 contests this choice and wins, this would mean that there exist $w'\in X$ and $w''\in N(w)\setminus X$ such that $\mathfrak{M},w'\sim^d_{\mathcal{Q}}\mathfrak{M},w''$. By the induction hypothesis, this would mean that $\mathfrak{M},w'\equiv_{\text{PL}(\mathcal{Q})}^d\mathfrak{M},w''$ and thus $\mathfrak{M},w''\Vdash \psi$, which contradicts the definition of $X$. We may thus assume that Player 2 does not contest the choice of $X$.

    Player 1 then chooses
    \begin{align*}
      &P := \{v' \in N(v) \mid\\ &\mathfrak{N},v'\Vdash \psi \text{ and } \mathfrak{N},v'\not\equiv^d_{\text{PL}(\mathcal{Q})}\mathfrak{M},w' \text{ for all } w'\in N(w)\}.  
    \end{align*}
    Again, if Player 2 contests this choice, then since $\mathfrak{N},v'\not\equiv^d_{\text{PL}(\mathcal{Q})}\mathfrak{M},w'$ for all $v'\in P$ and $w' \in N(w)$, by the induction hypothesis, Player 1 has a winning strategy.

    Player 2 must thus choose a witness set $X'$ of their own. Since $\mathfrak{N},v\Vdash \neg \GQ \psi$, we know that
    \[
        (N(v), \{v' \in N(v) \mid \mathfrak{N},v' \Vdash \psi\}) \not\in Q.
    \]
    Hence $X' \neq \{v' \in N(v) \mid \mathfrak{N},v' \Vdash \psi\}$. This means that either there exists a $v' \in X'$ such that $\mathfrak{N},v' \Vdash \neg \psi$ or there exists a $v' \in N(v)\setminus X'$ such that $\mathfrak{N},v' \Vdash \psi$.
    \begin{enumerate}
        \item Consider the first case. Since $\mathfrak{N},v'\Vdash \neg \psi$, we have $v'\in X'\setminus P$. Player 1 now simply chooses $v'$, and Player 2 is forced to respond by choosing $w'\in X\cup P$ such that $\mathfrak{M},w'\Vdash \psi$. Since $\mathfrak{M},w'\not\equiv^d_{\mathrm{PL}(\mathcal{Q})} \mathfrak{N},v'$, by the induction hypothesis, we conclude that the attacker, i.e. Player 1, has a winning strategy from this position.
        \item Consider then the second case. Since $P\subseteq X'$, we have $v'\notin P$ and thus $\mathfrak{N},v'\equiv^d_{\text{PL}(\mathcal{Q})}\mathfrak{M},w'$ for some $w'\in N(w)$. Hence $\mathfrak{M},w'\Vdash \psi$, so $w'\in X$. Player 1 now chooses $v'$ and $w'$ and becomes the defender. By the induction hypothesis, we have $\mathfrak{N},v'\sim^d_{\mathcal{Q}}\mathfrak{M},w'$, which means the defender, i.e. Player 1, has a winning strategy from this position.
    \end{enumerate}
    Player 1 therefore has a winning strategy for the $d+1$-round $\mathcal{Q}$-bisimulation game played on $(\mathfrak{M},w)$ and $(\mathfrak{N},v)$, which means $\mathfrak{M},w\not\sim^{d+1}_{\mathcal{Q}}\mathfrak{N},v$.
\end{proof}

\begin{lemma}\label{equiv-implies-wl}
    $\mathfrak{M},w\equiv^{d+1}_{\mathrm{PL}(\mathcal{Q})}\mathfrak{N},v\implies \mathfrak{N},v\Vdash \varphi^{\mathcal{Q},d+1}_{\mathfrak{M} \sqcup \mathfrak{N},w}$.
\end{lemma}
\begin{proof}
    Let \(N\) be the maximum degree of a node that belongs to either \(\mathfrak{M}\) or \(\mathfrak{N}\). By Theorem \ref{thm:saturation}, there is a finite set \(\mathcal{Q}_0 \subseteq \mathcal{Q}\) such that \(\varphi^{\mathcal{Q},d+1}_{\mathfrak{M} \sqcup \mathfrak{N},w}\) is equivalent with \(\varphi^{\mathcal{Q}_0,d+1}_{\mathfrak{M} \sqcup \mathfrak{N},w}\) over Kripke models where the degree of each node is at most \(N\). The claim then follows from the fact that \(\varphi_{\mathfrak{M}\sqcup \mathfrak{N},w}^{\mathcal{Q}_0,d+1}\) is a formula of PL(\(\mathcal{Q}\))\(^{d+1}\).
\end{proof}

\begin{lemma}\label{wl-implies-bisimulation}
    $\mathfrak{N},v\Vdash \varphi^{\mathcal{Q},d+1}_{\mathfrak{M} \sqcup \mathfrak{N},w} \implies \mathfrak{M},w\sim^{d+1}_{\mathcal{Q}}\mathfrak{N},v$.
\end{lemma}
\begin{proof}

    Suppose that $\mathfrak{N},v\Vdash \varphi^{\mathcal{Q},d + 1}_{\mathfrak{M} \sqcup \mathfrak{N},w}$. We now have two cases depending on which model Player 1 chooses the subset $X$ from.
    \begin{enumerate}
        \item In the first case Player 1 chooses a quantifier $Q \in \mathcal{Q}$, a witness set $X\subseteq N(w)$ such that $(N(w),X)\in Q$, and a spillover set $P\subseteq N(v)$.
        
        Let
        \[I := \{\varphi_{\mathfrak{M} \sqcup \mathfrak{N},w'}^{\mathcal{Q},d} \mid w' \in X\} \cup \{\varphi_{\mathfrak{M} \sqcup \mathfrak{N},v'}^{\mathcal{Q},d} \mid v' \in P\}.\]
        Because of the contestation rules, $X$ must be closed under \(\equiv_{\mathrm{PL}(\mathcal{Q})}^d\) and $P$ cannot contain a node equivalent to a node in $N(w)$. Hence \(\langle Q \rangle (\bigvee_{\varphi \in I} \varphi)\) occurs in \(\varphi_{\mathfrak{M} \sqcup \mathfrak{N},w}^{\mathcal{Q},d+1}\) and thus \(\mathfrak{N},v \Vdash \langle Q \rangle( \bigvee_{\varphi \in I} \varphi)\). Therefore Player 2 can choose the witness set 
        \begin{align*}
           &X':=\{v'\in N(v) \mid\mathfrak{N},v'\Vdash \varphi\text{ for some } \varphi\in I\}.
        \end{align*}

        Now, suppose Player 1 picks the first option, choosing $v'\in N(v)\setminus X'$ and $w'\in X$ and becoming the defender. Since $X'$ contains all the nodes in $N(v)$ that satisfy some $\varphi^{\mathcal{Q},d}_{\mathfrak{M}\sqcup \mathfrak{N},w}$ realized by a node $w\in X$, we conclude that $\mathfrak{N},v'\not\equiv^d_{\text{PL}(\mathcal{Q})}\mathfrak{M},w'$. By Lemma \ref{bisimulation-iff-equiv}, this implies that $\mathfrak{N},v'\not\sim^d_{\mathcal{Q}}\mathfrak{M},w'$, and Player 2 thus has a winning strategy.

        Suppose then that Player 1 picks the second option and chooses $v' \in X'$. If $\mathfrak{N},v'\Vdash \varphi^{\mathcal{Q},d}_{\mathfrak{M} \sqcup \mathfrak{N},v''}$ for some $v''\in P$, then Player 2 responds by choosing $v''$. Then, by the induction hypothesis $\mathfrak{N},v'\sim^d_\mathcal{Q}\mathfrak{N},v''$ and Player 2 thus has a winning strategy. If $\mathfrak{M},v'\Vdash \varphi^{\mathcal{Q},d}_{\mathfrak{M} \sqcup \mathfrak{N},w'}$ for some $w'\in X$, then Player 2 responds by choosing $w'$. Again, by the induction hypothesis $\mathfrak{N},v'\sim^d_\mathcal{Q}\mathfrak{M},w'$ and Player 2 thus has a winning strategy.

        \item In the second case, Player 1 chooses a quantifier $Q \in \mathcal{Q}$, a witness set $X\subseteq N(v)$ such that $(N(v),X)\in Q$, and a spillover set $P\subseteq N(w)$.
        Let 
        \[I := \{\varphi_{\mathfrak{M} \sqcup \mathfrak{N},v'}^{\mathcal{Q},d} \mid v' \in X\} \cup \{\varphi_{\mathfrak{M} \sqcup \mathfrak{N},w'}^{\mathcal{Q},d} \mid w' \in P\}.\]
        Now \(\mathfrak{N},v \Vdash \langle Q \rangle (\bigvee_{\varphi \in I} \varphi)\) and hence we must have that \(\mathfrak{M},w \Vdash \langle Q \rangle (\bigvee_{\varphi \in I} \varphi)\), because \(\mathfrak{N},v \Vdash \varphi_{\mathfrak{M} \sqcup \mathfrak{N},w}^{\mathcal{Q},d}\). The rest of the proof is similar to that of Case 1.
        
    \end{enumerate}
    Player 2 thus has a winning strategy for the $d+1$-round $\mathcal{Q}$-bisimulation game played on $(\mathfrak{M},w)$ and $(\mathfrak{N},v)$, which means $\mathfrak{M},w\sim^{d+1}_\mathcal{Q}\mathfrak{N},v$.
\end{proof}

\section{Effectively Finite Sets of Quantifiers}\label{appendix:effectively-finite}

In this section, we formally define what it means for a set of generalized modalities to be effectively finite. We start by defining this notion for sets of generalized quantifiers.

Let \(Q,Q'\) be generalized quantifiers. We say that they are \textbf{equivalent at width} \(N\), if for every pair \((A,P)\), where \(|A|\leq N\), we have that
\[(A,P) \in Q \Leftrightarrow (A,P) \in Q'.\]
Let \(\mathcal{Q}\) be a set of generalized quantifiers and \(\mathcal{Q}_0 \subseteq \mathcal{Q}\). We say that \(\mathcal{Q}_0\) is \textbf{a representation of} \(\mathcal{Q}\) \textbf{at width} \(N\), if for every \(Q \in \mathcal{Q}\) there exists a \(Q' \in \mathcal{Q}_0\) such that \(Q\) and \(Q'\) are equivalent at width \(N\). Note that if \(Q,Q'\) are equivalent at width \(N\), then the corresponding modalities \(\langle Q \rangle,\langle Q' \rangle\) intuitively ``behave similarly'' at nodes of out-degree at most \(N\).

\begin{theorem}\label{thm:saturation}
    Let $\mathcal{Q}$ be a set of generalized quantifiers. Then there exists a function $f:\mathbb{N}\to\mathcal{P}_\text{finite}(\mathcal{Q})$ such that $f(N)$ is a representation of $\mathcal{Q}$ at width $N$.
\end{theorem}

\begin{proof}
    A generalized quantifier is an isomorphism closed collection of pairs $(A, P)$, where \(P \subseteq A\). Since generalized quantifiers are closed under isomorphism, whether or not it contains a pair \((A,P)\) depends only on the isomorphism type of \((A,P)\) which is completely determined by the cardinalities of \(A\) and \(P\). If \(|A| = k\), then there are \(k + 1\) possible isomorphism types for \((A,P)\), corresponding to all possible sizes for \(P \subseteq A\). Therefore there are
    \[\prod_{k = 1}^N 2^{k+1} = 2^{\binom{N+1}{2} + N}\]
    non-equivalent quantifiers at width \(N\). Thus we can select $\mathcal{Q}_0\subseteq \mathcal{Q}$ such that $\mathcal{Q}_0$ is a representation of $\mathcal{Q}$ at width $N$ and $|\mathcal{Q}_0| \leq 2^{\binom{N+1}{2} + N}$.
\end{proof}

We say that a set \(\mathcal{Q}\) is \textbf{effectively finite} if it has the following properties.
\begin{enumerate}
    \item Each member \(Q\) of \(\mathcal{Q}\) is a Turing machine that computes the corresponding generalized quantifier. (Note that we use the same symbol for the quantifier and the corresponding Turing machine.) This means that when given a pair \((A,P)\) as input, \(Q\) decides whether it belongs to the corresponding generalized quantifier. Such generalized quantifiers are called \textbf{computable}.
    \item There exists a \emph{computable} function
    \[f:\mathbb{N} \to \mathcal{P}_{\text{finite}}(\mathcal{Q}),\]
    where \(\mathcal{P}_{\text{finite}}(\mathcal{Q})\) is the set of all finite subsets of \(\mathcal{Q}\), such that \(f(N)\) is a representation of \(\mathcal{Q}\) at width \(N\). Such a function is called an \textbf{effective representation} of \(\mathcal{Q}\).
\end{enumerate}

\begin{example}
    Consider the set
    \[\mathcal{Q}_{\text{graded}} := \{\exists^{\geq k} \mid k \in \mathbb{N}\}\]
    of all counting quantifiers. This set is effectively finite, because each quantifier \(\exists^{\geq k}\) is clearly computable and the function
    \[f(N) := \{\exists^{\geq 0},\exists^{\geq 1},\dots,\exists^{\geq N + 1}\}.\]
    is an effective representation of \(\mathcal{C}\). Indeed, for every pair \((A,P)\) with \(|A| \leq N\) we have that \(\exists^{\geq M}\) is equivalent with \(\exists^{\geq M'}\) for every \(M,M' > N\).
\end{example}

Recall that as an input, the $\mathcal{Q}$-WL-algorithm receives a Kripke model $\mathfrak{M}$ and a depth $d \in \mathbb{N}$. Of course, if $\mathcal{Q}$ is infinite, then $\mathcal{Q}$-WL is not finitary (i.e., not an algorithm). For the case where $\mathcal{Q}$ is effectively finite, we generalize $\mathcal{Q}$-WL to be finitary as follows.
\begin{enumerate}
    \item $\mathcal{Q}$-WL receives as an input a Kripke model $\mathfrak{M}$ and a depth $d \in \mathbb{N}$.
    \item We compute the maximum out-degree of $\mathfrak{M}$ and store it as $k$.
    \item Because \(\mathcal{Q}\) is effectively finite, it has an effective representation \(f:\mathbb{N} \to \mathcal{P}_{finite}(\mathcal{Q})\). We can thus compute the finite set of computable quantifiers $\mathcal{Q}_{0}\subseteq \mathcal{Q}$ associated with the effective representation $f$ such that $\mathcal{Q}_{0}:=f(k)$.
    \item We run $\mathcal{Q}_{0}$-WL on $\mathfrak{M}$ and $d$.
\end{enumerate}

Note that if we consider a set of graphs, as we do in the experimental section of the paper, then the input to the algorithm has to be the disjoint union of those graphs (to get the maximum out-degree of all graphs). Alternatively, we could consider a \emph{set} of graphs as an input to the algorithm and take the maximum out-degree of the graph set.

\section{Finite-Width Quantifiers}\label{appendix:finite-width}

In this section, we generalize Theorem \ref{thm:generalized-quantifiers-characterization} to cover all finite-width quantifiers. We also note that the proof trivially extends to cover the case of generalized quantifiers over arbitrary binary relations.

A \textbf{generalized quantifier of width} $n$ is an isomorphism-closed class of structures
\[
    (D,P_1,\dots,P_n)
\]
with domain $D$ and unary relations \(P_1,\dots,P_n \subseteq D\). Each generalized quantifier $Q$ of width $n$ gives rise to a a \textbf{generalized modality} $\GQ$ \textbf{of width} $n$, and a formula of the form $\GQ\varphi$ is interpreted in a pointed Kripke model such that 
\begin{align*}
    &\mathfrak{M},v\Vdash \GQ (\varphi_1,\dots,\varphi_n) \\&\iff(N(v),\varphi_1^{N(v)},\dots,\varphi_n^{N(v)})\in Q,
\end{align*}
where $\varphi_i^{N(v)}:=\{v'\in N(v)\mid \mathfrak{M},v'\Vdash \varphi_i\}$ for all $1\leq i\leq n$.

Intuitively, the $\mathcal{Q}$-bisimulation game can be generalized to finite-width quantifiers by simply replacing choosing sets with choosing tuples of sets. More precisely, the $\mathcal{Q}$-bisimulation game with finite-width quantifiers proceeds from the position $(v,v')$ as follows:
\begin{enumerate}
    \item The attacker wins immediately if $v$ and $v'$ do not satisfy the same proposition symbols.
    \item The attacker first chooses a quantifier $Q \in \mathcal{Q}$ and one of the two positions; without loss of generality, assume the attacker chooses $v$. If $Q$ is of width $n$, he then chooses $n$ witness sets $(X_i)_{1\leq i\leq n}$ such that $(N(w),X_1,\dots,X_n)\in Q$ and for each $X_i$ he selects a spillover set $P_i\subseteq N(v)$.
    \item The defender chooses $n$ witness sets $(X'_i)_{1\leq i\leq n}$ such that $(N(v),X'_1,\dots,X'_n)\in Q$ and $P_i\subseteq X'_i$ for all $1\leq i\leq n$.
    \item The attacker chooses one of the following. \begin{itemize}
        \item For some $1\leq i \leq n$, choose $u\in N(v')\setminus X'_i$ and $u'\in X_i$. Swap attacker and defender roles between the players, then a new round begins from the position $(u,u')$.
        \item For some $1\leq i \leq n$, choose $u\in X'_i$, after which the defender chooses $u'\in X_i\cup P_i$.
    \end{itemize}
    \item Whenever a tuple $(X_i)_{1\leq i\leq n}$ (resp. $(X'_i)_{1\leq i\leq n}$) of sets is chosen, the other player may contest any individual $X_i$ (resp. $X'_i$) in the usual way by placing a pebble onto a node in the contested set and the other onto a node in the complement $N(w)\setminus X_i$ (resp. $N(v)\setminus X'_i$), then taking the role of the defender. The next round then starts with the attacker's move.
\end{enumerate}

At any point after a set is chosen, the opposing player may instead contest it as follows:
\begin{itemize}
    \item After any witness set $Y_i\subseteq N(z)$ is chosen, contest that it breaks equivalence by choosing $u\in Y_i$ and $u'\in N(z) \setminus Y_i$. If the contesting player is the attacker, swap attacker and defender roles between the players. Then a new round begins from the position $(u,u')$.
    \item After any spillover set $P_i\subseteq N(v')$ is chosen, contest that it contains a type realized in $N(v)$ by choosing $u\in P_i$ and $u'\in N(v)$. Then a new round begins from the position $(u,u')$.
\end{itemize}

Next we extend the definition of \(\varphi_{\mathfrak{M},w}^{\mathcal{Q},d}\) to the case where \(\mathcal{Q}\) is an arbitrary set of finite-width quantifiers. For each \(w \in \mathfrak{M}\) we define \(\varphi_{\mathfrak{M},w}^{\mathcal{Q},d}\) recursively as follows. First, $\varphi^{\mathcal{Q},0}_{\mathfrak{M},w}$ is defined exactly as before. Then, having defined $\varphi^{\mathcal{Q},d}_{\mathfrak{M},w}$, we define
\begin{align*}
    & \varphi_{\mathfrak{M},w}^{\mathcal{Q},d+1} := \varphi_{\mathfrak{M},w}^{\mathcal{Q},0} \\
    & \land \\
    & \bigwedge_{\begin{array}{c}
        Q \in \mathcal{Q},\, C_1,\dots,C_m \subseteq \{\varphi^{\mathcal{Q},d}_{\mathfrak{M},w}\mid w \in \mathfrak{M}\} \\
        \mathfrak{M},w \Vdash \langle Q \rangle (\bigvee\limits_{\varphi \in C_1} \varphi, \dots, \bigvee\limits_{\varphi \in C_m} \varphi)
    \end{array}} \\ &\langle Q \rangle (\bigvee\limits_{\varphi \in C_1} \varphi, \dots, \bigvee\limits_{\varphi \in C_m} \varphi) \\
    & \land \\
    & \bigwedge_{\begin{array}{c}
        Q \in \mathcal{Q},\, C_1,\dots,C_m \subseteq \{\varphi^{\mathcal{Q},d}_{\mathfrak{M},w}\mid w \in \mathfrak{M}\} \\
        \mathfrak{M},w \Vdash \neg\langle Q \rangle (\bigvee\limits_{\varphi \in C_1} \varphi, \dots, \bigvee\limits_{\varphi \in C_m} \varphi)
    \end{array}} \\ &\neg\langle Q \rangle (\bigvee\limits_{\varphi \in C_1} \varphi, \dots, \bigvee\limits_{\varphi \in C_m} \varphi).
\end{align*}

We now prove Theorem \ref{thm:generalized-quantifiers-characterization} for finite-width quantifiers. First, we note that the base case is identical to the one in Appendix \ref{generalized-quantifiers-proof}. Then, as before, we assume that 1., 2. and 3. are equivalent for $d$. We prove their equivalence for $d+1$ via the following three Lemmas.

\begin{lemma}
    Lemma \ref{bisimulation-iff-equiv} holds when $\mathcal{Q}$ is a set of finite-width generalized modalities.
\end{lemma}
\begin{proof}
    
    Player 1's strategy is to choose
    \[
        X_i:=\{w'\in N(w)\}\mid \mathfrak{M},w'\Vdash\psi_i\}
    \]
    and
    \begin{align*}
        &P_i:=\{v'\in N(v)\mid
        \\& \mathfrak{N},v'\Vdash \psi_i \text{ and } \mathfrak{N},v'\not\equiv^d_{\text{PL}(\mathcal{Q})}\mathfrak{M},w'\text{ for all } w'\in N(w)\}
    \end{align*}
    for all $1\leq i \leq n$. By the same arguments as in Lemma \ref{bisimulation-iff-equiv}, Player 2 cannot win by contesting any of these sets.

    Player 2 must thus choose an $n$-tuple $(X_i)_{1\leq i\leq n}$ of his own. Since $\mathfrak{N},v\Vdash \neg \GQ(\psi_1,\dots,\psi_n)$, we know that
    \[
        (N(v), \psi_1^{N(v)},\dots,\psi_n^{N(v)})\notin Q.
    \]
    Hence $X'_i\neq\{v'\in N(v) \mid \mathfrak{N},v'\Vdash\psi_i\}$ for some $1\leq i \leq n$. Player 1's winning strategy is now similar to that in Lemma \ref{bisimulation-iff-equiv}.
\end{proof}

\begin{lemma}
    Lemma \ref{equiv-implies-wl} holds when $\mathcal{Q}$ is a set of finite-width generalized modalities.
\end{lemma}
\begin{proof}
    Similarly to Lemma \ref{equiv-implies-wl} this follows from Theorem \ref{thm:saturation} and the fact that for finite \(\mathcal{Q}\) the formula \(\varphi_{\mathfrak{M}\sqcup \mathfrak{N},w}^{\mathcal{Q},d+1}\) is a formula of PL(\(\mathcal{Q}\))\(^{d+1}\).
\end{proof}

\begin{lemma}
    Lemma \ref{wl-implies-bisimulation} holds when $\mathcal{Q}$ is a set of finite-width generalized modalities.
\end{lemma}
\begin{proof}

    Suppose Player 1 chooses a quantifier $\GQ$ of width $n$ and the $n$-tuples $(X_i)_{1\leq i\leq n}$ and $(P_i)_{1\leq i\leq n}$. Let
    \begin{align*}
        I_i:=\{\varphi^{\mathcal{Q},d}_{\mathfrak{M} \sqcup \mathfrak{N},w'}\mid w'\in X_i\}\cup\{\varphi^{\mathcal{Q},d}_{\mathfrak{M} \sqcup \mathfrak{N},v'}\mid v'\in P_i\}
    \end{align*}
    for all $1 \leq i \leq n$. Because of the contestation rules, each $X_i$ must be closed under \(\equiv_{\mathrm{PL}(\mathcal{Q})}^d\) and no $P_i$ can contain a node equivalent to a node in $N(w)$. Hence 
    \[
        \langle Q \rangle (\bigvee\limits_{\varphi \in I_1} \varphi, \dots, \bigvee\limits_{\varphi \in I_n} \varphi)
    \]
    occurs in $\varphi^{\mathcal{Q},d+1}_{\mathfrak{M}\sqcup \mathfrak{N},v}$, and thus
    \begin{align*}
        &\mathfrak{N},v\Vdash\langle Q \rangle (\bigvee\limits_{\varphi \in I_1} \varphi, \dots, \bigvee\limits_{\varphi \in I_n} \varphi).
    \end{align*}
    Therefore Player 2 can choose the witness sets
    \[
        X_i:=\{v'\in N(v)\mid \mathfrak{N},v'\Vdash \varphi_i \text{ for some }\varphi_i\in I_i\}
    \]
    for all $1 \leq i \leq n$. Player 2's winning strategy is now similar to that in Lemma \ref{wl-implies-bisimulation}.
\end{proof}

From these Lemmas, we immediately acquire the following generalization of Theorem \ref{thm:generalized-quantifiers-characterization}.
\begin{theorem}
    Theorem \ref{thm:generalized-quantifiers-characterization} holds when $\mathcal{Q}$ is a set of finite-width generalized quantifiers.
\end{theorem}


\end{document}